\documentclass{article}

    \usepackage[final]{neurips_2023}

\usepackage{soul}
\usepackage{hyperref}
\usepackage{url}
\usepackage{booktabs}       % professional-quality tables
\usepackage{amsfonts}       % blackboard math symbols
\usepackage{nicefrac}       % compact symbols for 1/2, etc.
\usepackage{microtype}      % microtypography
\usepackage{xcolor}         % colors

\usepackage[utf8]{inputenc} % allow utf-8 input
\usepackage[T1]{fontenc}    % use 8-bit T1 fonts
\usepackage{hyperref}       % hyperlinks
\usepackage{url}            % simple URL typesetting
\usepackage{booktabs}       % professional-quality tables
\usepackage{amsfonts}       % blackboard math symbols
\usepackage{nicefrac}       % compact symbols for 1/2, etc.
\usepackage{microtype}      % microtypography
\usepackage{xcolor}         % colors

\usepackage{graphicx}
\usepackage{amsthm}
\usepackage{amssymb}
\usepackage{amsmath}
\usepackage{mathtools}
\usepackage{dirtytalk}
\usepackage{subcaption}
\usepackage{wrapfig}
\usepackage{multirow}
\usepackage{placeins}
\usepackage[noend,ruled,vlined,linesnumbered]{algorithm2e}  % algo2e,

\SetCommentSty{mycommfont}

\usepackage{color}
\usepackage{soul}
\usepackage{ulem}
\usepackage[capitalise]{cleveref}

\theoremstyle{definition}
\newtheorem{definition}{Definition}%[section]
\newtheorem{assumption}{Assumption}
\newtheorem{problem}{Problem}
\newtheorem{theorem}{Theorem}
\newtheorem{algo}{Algorithm}

\newtheorem{proposition}{Proposition}%[section]
\ifdefined\cleanversion

    \newcommand{\rd}[1]{}
    \newcommand\ido[1]{}
\else

    \newcommand{\rd}[1]{{\textcolor{red}{\sout{{#1}}}}}
    \definecolor{red}{rgb}{0.95, 0.5, 0.5}
    \newcommand\ido[1]{\textcolor{red}{\textbf{IG:} #1 }}
\fi

\newcommand{\pypibase}{https://pypi.org/project/Optimized-Kalman-Filter/} %{https://github.com/...}
\newcommand{\repobase}{https://github.com/ido90/UsingKalmanFilterTheRightWay}
\newcommand{\pypi}[1]{\href{\pypibase}{\underline{#1}}}
\newcommand{\repo}[1]{\href{\repobase}{\underline{#1}}}

\newcommand{\mytitle}{Optimization or Architecture:\\How to Hack Kalman Filtering}

\title{\mytitle}

\author{
  Ido Greenberg \\ %\thanks{Use footnote for providing further information about author (webpage, alternative address)---\emph{not} for acknowledging funding agencies.} \\
  Technion \\
  \texttt{gido@campus.technion.ac.il}
  \And
  Netanel Yannay \\
  ELTA Systems \\
  \texttt{natiy4@gmail.com}
  \And
  Shie Mannor \\
  Technion, Nvidia Research \\
  \texttt{shie@ee.technion.ac.il}
}

\begin{document}

\maketitle

% \vspace{-5pt}
\begin{abstract}
In non-linear filtering, it is traditional to compare non-linear architectures such as neural networks to the standard linear Kalman Filter (KF).
We observe that this mixes the evaluation of two separate components: the non-linear architecture, and the parameters optimization method.
In particular, the non-linear model is often optimized, whereas the reference KF model is not.
% To evaluate a non-linear architecture against a linear one, 
We argue that \textit{both} should be optimized similarly, and to that end present the Optimized KF (\textbf{OKF}).
% We suggest the Optimized KF (\textbf{OKF}), which adjusts numeric optimization to the positive-definite KF parameters.
We demonstrate that the KF may become competitive to neural models -- if optimized using OKF.
% We demonstrate how a significant advantage of a neural network over the KF may \textit{entirely vanish} once the KF is optimized using OKF.
This implies that experimental conclusions of certain previous studies were derived from a flawed process.
% While using the same linear architecture, as the KF, we study its advantage over the KF theoretically and empirically.
The advantage of OKF over the standard KF is further studied theoretically and empirically, in a variety of problems.
Conveniently, OKF can replace the KF in real-world systems by merely updating the parameters.
Our experiments are published in \repo{Github}, and the OKF in \pypi{PyPI}.
\end{abstract}

%%%%%%%%%%%%%%%%%%%%%%%%%%%%%%%%%%%%%%%%%%%%%%%%%
%%%%%%%%%%%%%%%%%%%%%%%%%%%%%%%%%%%%%%%%%%%%%%%%%

\section{Introduction}
\label{sec:intro}

The Kalman Filter (KF)~\citep{KF} is a celebrated method for linear filtering and prediction, with applications in many fields including tracking, navigation, control and reinforcement learning \citep{KF_practical_guide,tracking_and_navigation,rl_drones}. % guidance
The KF provides optimal predictions under certain assumptions (namely, linear models with i.i.d noise).
In practical problems, these assumptions are often violated, rendering the KF sub-optimal and motivating the growing field of non-linear filtering.
Many studies demonstrated the benefits of non-linear models over the KF \citep{KalmanNet,pose_estimation}.

% ,bai2020neuron,navigation_using_RNN,ManeuveringTargetTracking2,ANN_vs_EKF,KF_vs_NN_for_batteries,learning_in_indoor_navigation,KF_with_ANN,RNN_EKF}.

Originally, we sought to join this line of works. Motivated by a real-world Doppler radar problem, we developed a dedicated non-linear Neural KF (NKF) based on the LSTM sequential model.
% NKF is based on the LSTM model, which is a key component in many SOTA algorithms for non-linear sequential prediction~\citep{process_prediction_review}.
NKF achieved significantly better accuracy than the linear KF. %, as presented in \cref{sec:nkf}.

Then, during ablation tests, we noticed that the KF and NKF differ in \textit{both architecture and optimization}. % between the NKF and the KF, effectively comparing ``apples and oranges''. %; i.e., both the model architecture and the method used to tune its parameters.
% To avoid comparing ``apples and oranges'',
Specifically, the KF's noise parameters are traditionally determined by noise estimation~\citep{ALS}; whereas NKF's parameters are optimized using supervised learning.
To fairly evaluate the two architectures, we wished to apply the same optimization to both. %, namely, sequential supervised learning.
% To apply the same supervised learning to the KF, whose parameters are positive-definite, we used the Cholesky parameterization.
To that end, we devised an Optimized KF (\textbf{OKF}, \cref{sec:okf}).
% To optimize the positive-definite KF parameters, we used the Cholesky parameterization.
% This resulted in an Optimized KF (\textbf{OKF}, \cref{sec:okf}).
KF and OKF have the same linear architecture:
% The sole difference, is that OKF considers the noise parameters as "something to be optimized".
OKF only changes the noise parameters values.
% on inference, the sole difference is the noise parameters values.
% To optimize the positive-definite KF parameters, we used the Cholesky parameterization along with standard sequential supervised learning.
% This Optimized KF (\textbf{OKF}) is introduced in \cref{sec:okf}.
% Although OKF has the same linear architecture as the standard KF
Yet, unlike KF, OKF \textit{outperformed} NKF, which reversed the whole experimental conclusion, and made the neural network unnecessary for this problem (\cref{sec:nkf}).
% , and \textit{eliminated} the need for the complicated NKF model.

Our original error was comparing
% Our original errorneous methodology compared
two different model architectures (KF and NKF) that were not optimized similarly.
A review of the non-linear filtering literature reveals that this methodology is used in many studies. %: two different model architectures are compared, but are not optimized similarly.
% many studies adopt a similar methodology to our original one: comparing two different model architectures that are not optimized in a similar manner.
Specifically, for a baseline KF model, the parameters are often tuned by noise estimation \citep{navigation_using_RNN,ANN_vs_EKF,KalmanNet}; by heuristics \citep{learning_in_indoor_navigation,pose_estimation,KF_with_ANN}; or are simply ignored \citep{ManeuveringTargetTracking2,bai2020neuron,RNN_EKF}, often without public code for examination.
% the complex, non-linear model \textit{may} be beneficial;
% % It is of course possible that the complex, non-linear models in these studies are beneficial;
% however, this conclusion cannot be inferred from the experimental evidence.
\citet{KF_vs_NN_for_batteries} even discusses the (Extended-)KF sensitivity to its parameters, and suggests a neural network with supervised learning -- yet never considers the same supervised learning for the KF itself.
% In most of the examples, no code is publicly available either.
In all these 10 studies, the non-linear model's added value cannot be inferred from the experimental evidence.

So far, OKF is presented as a \textit{methodological contribution}: for comparison with non-linear methods, it forms a more coherent baseline than KF.
% it is closer than KF to standard non-linear methods, and can be used as a reliable baseline for comparison.
In addition, OKF provides a \textit{practical contribution}: it achieves more accurate filtering than the KF, using identical architecture.
% outperforms the KF and provides more accurate filtering.
% The superior accuracy of OKF is consistently demonstrated
This is demonstrated extensively in \cref{sec:experiments} and \cref{app:okf} -- in different domains, %  (radar, video and lidar)
over different problem variations, using different KF baselines, with different data sizes, and even under distributional shifts.
% over different violations of KF assumptions
% All our experiments are available on \repo{Github}.

% OKF is not just an experimental baseline, but also a filtering algorithm in its own right.
% We extensively compare OKF to the KF in \cref{sec:experiments} and \cref{app:okf} -- over different violations of KF assumptions, in different domains, %  (radar, video and lidar)
% using different KF baselines, with different amounts of train data, and even under distributional shifts.
% In \textit{all} of these settings, OKF consistently achieves superior test accuracy. % compared to the standard KF.
% % as well as lower sensitivity to the KF configuration (e.g., to the choice of coordinate system).

% , which is available on \pypi{PyPI},
%  (link removed to maintain double blindness)
\textbf{Discrepancy of objectives}:
The advantage of OKF over KF may come as a surprise:
KF's standard noise estimation is known to already obtain the MSE-optimal parameters!
% when tuning the linear architecture's parameters, standard noise estimation is known to be already MSE-optimal!
% for their linear architecture, tuning the parameters via KF's standard noise estimation is known to be already MSE-optimal!
% Notice that OKF outperforms KF using the same linear architecture. This may come as a surprise: for this architecture, tuning the KF by standard noise estimation is known to be already MSE-optimal!
However, this optimality relies on unrealistic assumptions, often considered \say{{fairy tales for undergraduates}} \citep{fairytales}.
When violated, a conflict emerges between noise estimation and MSE optimization. We study this in detail:
% leaving room for optimization by OKF. %, and their violation may be difficult to even notice. % commonly
% We study the emerged conflict between noise estimation and MSE optimization; for example:
(a) \cref{sec:experiments} analyzes the conflict theoretically under certain assumption violations; (b) \cref{app:okf_detailed} shows that even oracle noise estimation cannot optimize the MSE; (c) \cref{app:okf_train_size} shows that
when using noise estimation, \textbf{the MSE may degrade with more data}.
% feeding more data to noise estimation may in fact \textit{degrade} the MSE.
% \cref{sec:experiments} analyzes two violations -- non-linear dynamics and non-i.i.d noise -- and theoretically explains the KF's sub-optimality.
% Disturbingly, such violations may even cause the KF errors to deteriorate with the train data size (\cref{app:okf_train_size}).
In this light, all our findings can be summarized as follows (also see \cref{fig:main}):

\textbf{Contribution:}
(a) We observe that in most scenarios, since the KF assumptions do not hold, noise estimation is \textit{not} a proxy to MSE optimization.
(b) We thus present the Optimized KF (OKF), also available as a \pypi{PyPI package}.
We analyze (theoretically and empirically) the consequences of neglecting to optimize the KF:
(c) the standard KF tuning method leads to sub-optimal predictions; (d) the standard methodology in the literature compares an optimized model to a non-optimized KF, hence may produce misleading conclusions. Note that we \textit{do not} argue against the non-linear models; rather, we claim that their added value cannot be deduced from flawed experiments.

% The sub-optimality of the KF may come as a surprise, as the standard method for KF tuning is known to be optimal under certain assumptions.
% % the fragility of these assumptions
% However, these assumptions are quite fragile, as discussed in \cref{sec:experiments}: their violation is not only common, but also difficult to notice in certain cases.
% We theoretically analyze the effect of two different violations -- non-linear model and non-i.i.d noise -- and show that a single violation may result in significant changes of the optimal parameters.
% In fact, since the standard KF no longer optimizes the filtering errors, the errors sometimes \textit{deteriorate with the amount of train data} (\cref{app:okf_train_size}).

% \begin{wrapfigure}[13]{}{0.52\textwidth} % I
\begin{figure}
% \centering
  % \vspace{-33pt}
  \begin{center}
    \includegraphics[width=0.9\linewidth]{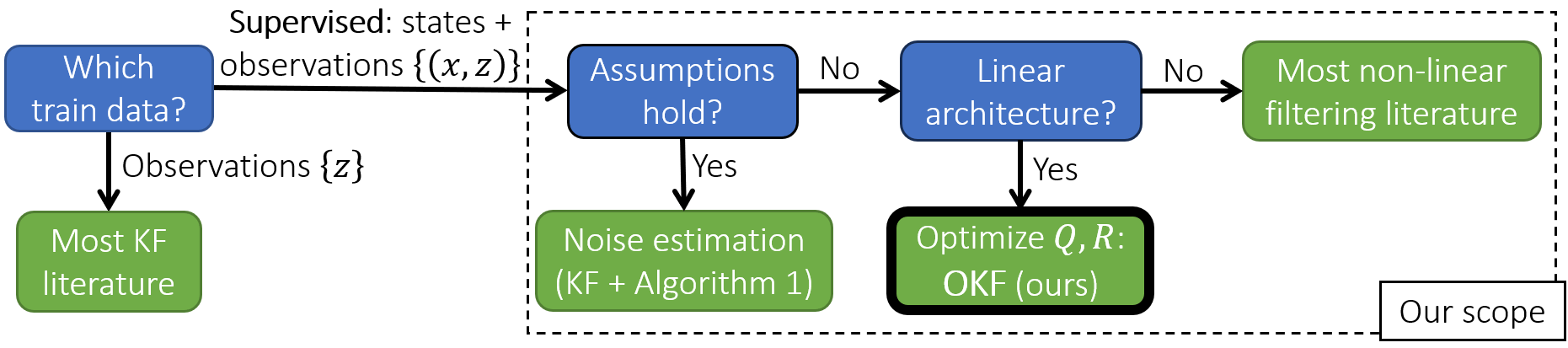}
  \end{center}
\caption{\small Since the KF assumptions are often violated, noise estimation does not optimize the MSE. Instead, our method (OKF) optimizes the MSE directly. In particular, neural network models should be tested against OKF rather than the non-optimized KF -- in contrast to the common practice in the literature.}
\label{fig:main}
\end{figure}
% \end{wrapfigure}

% The optimized noise parameters carry interpretable physical meaning regarding the \textit{effective} noise. For example, if the observation model is non-linear, the \textit{projection} of the observation onto the state space becomes noisy.
% As analyzed both theoretically and empirically in \cref{sec:experiments}, this increases the effective noise, as it accumulates on top of the inherent observation uncertainty.

% For the two assumptions of linearity and i.i.d noise, we analyze theoretically the effect of their violation, and show that even small violations may change the optimal parameters significantly.

% also mention winning the oracle baseline in \cref{app:okf_detailed}.?
% In addition to the empirical evidence, we show theoretically that even small violations may cause a significant deviation of the optimal parameters.

\textbf{Scope:}
\textbf{We focus on the supervised filtering setting},
%, which is widely studied in non-linear filtering, yet is somewhat overlooked by linear Kalman filtering.
% Similarly to the studies cited above, our work focuses on the \textit{supervised} filtering setting.
% Under this setting,
where training data includes both observations and the true system states (whose prediction is usually the objective). % are available in the training data. %, thus their prediction can be learned in a supervised manner.
Such data is available in many practical applications.
% This setting is commonly encountered in practical applications.
For example, the states may be provided by external accurate sensors such as GPS \citet{navigation_using_RNN}; %(e.g., GPS as approximated ground-truth for inertial measurements, \citet{navigation_using_RNN});
by manual object labeling in computer vision \citep{deep_sort}; by controlled experiments of radar targets; or by simulations of known dynamic systems.
% Such data is often available from controlled experiments, simulations or manual labeling.

As demonstrated in the 10 studies cited above, this supervised setting is common in non-linear filtering.
In linear Kalman filtering, this setting seems to be solved by trivial noise estimation; thus, the literature tends to overlook it, and instead focuses on settings that do not permit trivial noise estimation, e.g., learning from observations alone.
Nevertheless, we argue that even in the supervised setting, noise estimation is often not a proxy to MSE optimization, and thus should often be avoided.

\section{Preliminaries}
\label{sec:preliminaries}

% The KF~\citep{KF,KF_fresh_look} relies on the following model for a dynamic system:
Consider the KF model for a dynamic system with no control signal \citep{KF}: %,KF_fresh_look}:
\begin{align}
\label{eq:KF_model}
\begin{split}
    X_{t+1}  = F_tX_t + \omega_t \quad  (\omega_t\sim \mathcal{N}(0,Q)), \qquad\quad Z_{t}  = H_tX_t + \nu_t \quad  (\nu_t\sim \mathcal{N}(0,R)) .
    % X_{t+1} &= F_tX_t + \omega_t \qquad (\omega_t\sim \mathcal{N}(0,Q)) \\           Z_{t} &= H_tX_t + \nu_t \qquad (\nu_t\sim \mathcal{N}(0,R)) .
\end{split}
\end{align}
$X_t$ is the system state at time $t$, and its estimation is usually the goal. Its dynamics are modeled by the linear operator $F_t$, with random noise $\omega_t$ whose covariance is $Q$. $Z_t$ is the observation, modeled by the operator $H_t$ with noise $\nu_t$ whose covariance is $R$.
The notation may be simplified to $F,H$ in the stationary case.
% The notation may be simplified to $F,H$ if the operators are not assumed to depend on time.

The KF represents $X_t$ via estimation of the mean $\hat{x}_t$ and covariance $\hat{P}_t$. % of a normal distribution.
As shown in \cref{fig:KF}, the KF alternately predicts the next state (\textit{prediction} step), and processes new information from incoming observations (\textit{update} or \textit{filtering} step). % using the dynamics model
The KF relies on the matrices $\tilde{F}_t,\tilde{H}_t,\hat{Q},\hat{R}$, intended to represent $F_t,H_t,Q,R$ of \cref{eq:KF_model}.
Whenever $F_t,H_t$ are known and stationary, we may simplify the notation to $\tilde{F}_t=F,\,\tilde{H}_t=H$.

% \begin{wrapfigure}[13]{}{0.52\textwidth} % I
\begin{figure}%[!t]
% \centering
  % \vspace{-33pt}
  \begin{center}
    \includegraphics[width=0.55\linewidth]{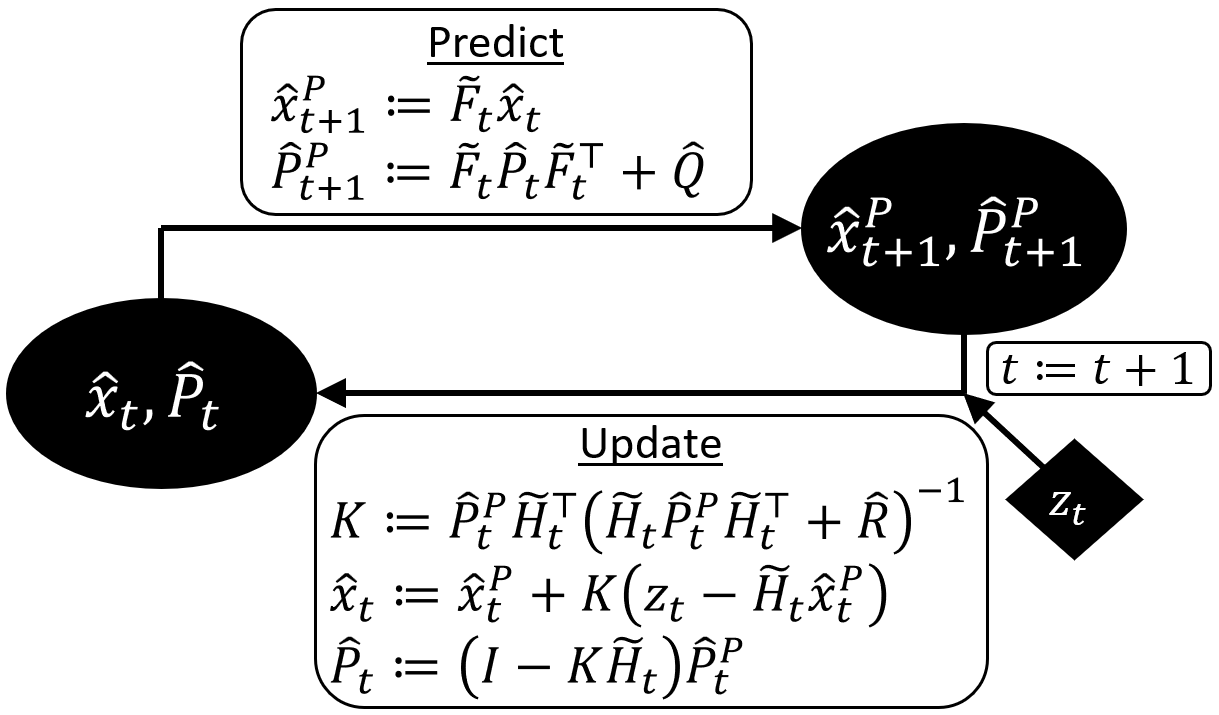}
  \end{center}
\caption{\small The KF algorithm. The prediction step is based on the motion model $\tilde{F}_t$ with noise $\hat{Q}$, whereas the update step is based on the observation model $\tilde{H}_t$ with noise $\hat{R}$.}
\label{fig:KF}
\end{figure}
% \end{wrapfigure}

The KF estimator $\hat{x}_t$ is optimal in terms of mean square errors (MSE) -- but only under a restrictive set of assumptions~\citep{KF}: %, as specified below.
% in Assumption~\ref{assumption:KF}.
% Note that normality of the noise is excluded: it would guarantee that the KF is the optimal model, but it is not necessary for optimality of the parameters \textit{within} a KF \citep{KF_fresh_look}.
% Note that normality of the noise is excluded since it is not necessary for optimality of the parameters~\citep{KF_fresh_look}.
%
\begin{assumption}[KF assumptions]
\label{assumption:KF}
% According to the assumptions of KF,
$\tilde{F}_t=F_t,\,\tilde{H}_t=H_t$ are known and independent of $X_t$ (linear models); each sequence $\{\omega_t\},\{\nu_t\}$ is i.i.d; the covariances $\hat{Q}=Q,\,\hat{R}=R$ are known; and $\hat{x}_0,\hat{P}_0$ correspond to the mean and covariance of the initial $X_0$. %distribution of the random initial state $X_0$ is known.
\end{assumption}
%
% \begin{theorem}[KF optimality; e.g., see Sections 3.2 and C in \citet{KF_fresh_look}]
\begin{theorem}[KF optimality; e.g., see \citet{jazwinski_book,KF_fresh_look}]
\label{theorem:optimality}
Under \cref{assumption:KF}, the KF estimator $\hat{x}_t$ minimizes the MSE w.r.t.~$X_t$.
% If $\omega_t,\nu_t$ are also normally-distributed, the KF is optimal among all models, not just linear ones.
\end{theorem}

The KF accuracy strongly depends on its parameters $\hat{Q}$ and $\hat{R}$~\citep{noise_cov_estimation}.
As motivated by \cref{theorem:optimality}, these parameters are usually identified with the noise covariance $Q,R$ and are set accordingly:
\say{the systematic and preferable approach to determine the filter gain is to estimate the covariances from data} \citep{ALS}.
In absence of system state data $\{x_t\}$ (the ``ground truth''), many methods were suggested to estimate the covariances from observations $\{z_t\}$ alone \citep{Mehra,cov_estimation_varying_processes,measurement_noise_recommendation,seq_cov_estimation}.
We focus on the supervised setting, where the states $\{x_t\}$ are available in the training-data (but not in inference).
\begin{definition}[Supervised data]
\label{def:supervised_data}
Consider $K$ trajectories of a dynamic system, with lengths $\{T_k\}_{k=1}^K$.
We define their supervised data as the sequences of true system states $x_{k,t} \in \mathbb{R}^{d_x}$ and observations $z_{k,t} \in \mathbb{R}^{d_z}$: $\ \{\{(x_{k,t},z_{k,t}) \}_{t=1}^{T_k} \}_{k=1}^{K}$.
\end{definition}
If $F_t,H_t$ are known, the supervised setting permits a direct calculation of the sample covariance matrices of the noise~\citep{TutorialKF}:
\vspace{-3pt}
\begin{align}
\label{eq:noise_estimation}
\begin{split}
    % R&\coloneqq Cov(\{z_t-Hx_t\}_t) \\ Q&\coloneqq Cov(\{x_{t+1}-Fx_t\}_t)
    % \hat{Q}&\coloneqq Cov(\{x_{k,t+1}-F_t x_{k,t}\}_{k,t}) \\
    % \hat{R}&\coloneqq Cov(\{z_{k,t}-H_t x_{k,t}\}_{k,t}) .
    \hat{Q}&\coloneqq Cov(\{x_{k,t+1}-F_t x_{k,t}\}_{k,t}), \qquad
    \hat{R}\coloneqq Cov(\{z_{k,t}-H_t x_{k,t}\}_{k,t}) .
\end{split}
\end{align}
%
% where $\{\{\cdot\}_{t=1}^{T_k} \}_{k=1}^{K}$ is the trajectories concatenation, with total length of $\sum_k T_k$.
Since \cref{theorem:optimality} guarantees optimality when $\hat{Q}=Q,\hat{R}=R$, and \cref{eq:noise_estimation} provides a simple estimator for $Q$ and $R$,
% Due to the simplicity of \cref{eq:noise_estimation} and the optimality guarantee of \cref{theorem:optimality},
\cref{algo:KF} has become the gold-standard tuning method for KF from supervised data.
\begin{algo}[KF noise estimation]
    \label{algo:KF}
    Given supervised data $\{(x_{k,t},z_{k,t})\}$, return $\hat{Q}$ and $\hat{R}$ of \cref{eq:noise_estimation}.
\end{algo}
While \cref{algo:KF} is indeed trivial to apply in the supervised setting, we show below that when \cref{assumption:KF} is violated, it no longer provides optimal predictions.
% However, while the parameters \textit{can} be estimated by \cref{eq:noise_estimation},
Violation of \cref{assumption:KF} can be partially handled by certain variations of the KF, such as Extended KF (EKF)~\citep{KF_theory} and Unscented KF (UKF)~\citep{UKF}.

\section{Optimized Kalman Filter}
\label{sec:okf}

Estimation of the KF noise parameters $\hat{Q},\hat{R}$ has been studied extensively in various settings; yet, in our supervised setting it is trivially solved by \cref{algo:KF}.
However, once \cref{assumption:KF} is violated, such noise estimation is no longer a proxy to MSE optimization -- despite \cref{theorem:optimality}.
Instead, in this section we propose to determine $\hat{Q}$ and $\hat{R}$ via explicit MSE optimization.
We rely on standard optimization methods for sequential supervised learning; as discussed below, the main challenge is to maintain the Symmetric and Positive Definite (SPD) structure of $\hat{Q},\hat{R}$ as covariance matrices.
Formally, we consider the KF (\cref{fig:KF}) as a prediction model $\hat{x}_{k,t}(\{z_{k,\tau}\}_{\tau=1}^t;\,\hat{Q},\hat{R})$, which estimates ${x}_{k,t}$ given the observations $\{z_{k,\tau}\}_{\tau=1}^t$ and parameters $\hat{Q},\hat{R}$.
We define the KF optimization problem:
\begin{align}
\label{eq:obj}
\begin{split}
    \underset{Q',R'}{\mathrm{argmin}} &\sum_{k=1}^K \sum_{t=1}^{T_k} \mathrm{loss}\left( \hat{x}_{k,t}\left(\{z_{k,\tau}\}_{\tau=1}^t;\,Q',R'\right),\ x_{k,t} \right), \qquad
    \text{s.t. } Q'\in S_{++}^{d_x},\ R'\in S_{++}^{d_z} ,
\end{split}
\end{align}
where $S_{++}^{d} \subset \mathbb{R}^{d\times d}$ is the space of Symmetric and Positive Definite matrices (SPD), and $\mathrm{loss}(\cdot)$ is the objective function (e.g., $\mathrm{loss}(\hat{x},x)=||\hat{x}-x||^2$ for MSE).
Prediction of future states can be expressed using the same \cref{eq:obj}, by changing the observed input from $\{z_{k,\tau}\}_{\tau=1}^{t}$ to $\{z_{k,\tau}\}_{\tau=1}^{t-1}$.

A significant challenge in solving \cref{eq:obj} is the SPD constraint.
Standard numeric supervised optimization methods (e.g., Adam \citep{adam}) may violate the constraint.
% Parameters of sequential prediction models can be optimized via standard numeric optimization methods (e.g., Adam \citep{adam}).
% To optimize \cref{eq:obj}, we wish to apply a numeric optimization method (e.g., Adam \citep{adam}) to the sequential prediction KF model.
% However, such methods may violate the SPD constraint of \cref{eq:obj}.
% In other settings, 
% In other, non-supervised settings, when the KF parameters are optimized,
While the SPD constraint is often bypassed using diagonal restriction \mbox{\citep{unsupervised_KF_tuning,noise_cov_estimation}}, this may significantly degrade the predictions, as demonstrated in the ablation tests in \cref{sec:diagonal}.
% the SPD constraint is often bypassed using diagonal restriction \mbox{\citep{unsupervised_KF_tuning}}, as pointed by \citet{noise_cov_estimation}:
% \say{\textit{since both the covariance matrices must be constrained to be positive semi-definite, $Q$ and $R$ are often parameterized as diagonal matrices}}.
Instead, to maintain the complete expressiveness of $\hat{Q}$ and $\hat{R}$, we use the Cholesky parameterization \citep{cov_parameterization}.

\begin{wrapfigure}[20]{L}{0.53\textwidth} % I
\vspace{-12pt}
% \hspace{3pt}
\begin{minipage}{\linewidth}
\setcounter{algocf}{1}
\begin{algorithm}[H]
% \begin{algorithm2e}
\caption{Optimized Kalman Filter (OKF)}
\label{algo:OKF}
% \setstretch{1.1}
\DontPrintSemicolon
\SetAlgoNoLine
\SetNoFillComment

 {\bf Input}: training data $\{(x_{k,t},z_{k,t})\}_{k=1}^K$ (\cref{def:supervised_data}); batch size $b$; $\mathrm{loss}$ function (e.g., MSE); $\mathrm{optimization\_step}$ function (e.g., Adam)\;
 \BlankLine
 \BlankLine
 $d_x \leftarrow \mathrm{len}(x_{1,1}), \quad d_z \leftarrow \mathrm{len}(z_{1,1})$\;
 Initialize $\theta_Q \in\mathbb{R}^{\frac{1}{2}d_x(d_x+1)},\ \theta_R \in\mathbb{R}^{\frac{1}{2}d_z(d_z+1)}$\;
 \While{training not finished}{
  \tcp{Get $Q,R$ using \cref{eq:L}}
  $\hat{Q} \leftarrow L(\theta_Q)L(\theta_Q)^\top, \quad \hat{R} \leftarrow L(\theta_R)L(\theta_R)^\top$\;
  $\mathcal{K} \leftarrow \mathrm{sample}(\{1,...,K\},\,\text{size=}b)$\;
  $C \leftarrow 0$\;
  \For{$k$ in $\mathcal{K}$}{
   Initialize $\hat{x} \in\mathbb{R}^{d_x}$\;
   \For{$t$ in $1:T_k$}{
   \tcp{KF steps (\cref{fig:KF})}
   $\hat{x} \leftarrow \mathrm{KF\_predict}(\hat{x};\, \hat{Q})$ \; \label{line:pred}
   $\hat{x} \leftarrow \mathrm{KF\_update}(\hat{x}, z_{k,t};\, \hat{R})$ \; \label{line:update}
   $C \leftarrow C + \mathrm{loss}(\hat{x},\,x_{k,t})$ \; \label{line:loss}
  }
 }
 $\theta_Q,\theta_R \leftarrow \mathrm{optimization\_step}(C,\, (\theta_Q,\theta_R))$ \label{line:optim}\;
 }
 Return $\hat{Q},\, \hat{R}$\;
\end{algorithm}
% \vspace{-5pt}
\end{minipage}
\end{wrapfigure}

% ~\citep{cholesky}
The parameterization relies on Cholesky decomposition: any SPD matrix $A\in \mathbb{R}^{d\times d}$ can be written as $A=LL^\top$, where $L$ is lower-triangular with positive entries along its diagonal.
% The reversed claim is also true:
Reversely, for any lower-triangular $L$ with positive diagonal, $LL^\top$ is SPD.
Thus, to represent an SPD $A \in \mathbb{R}^{d\times d}$, we define $A(L) \coloneqq LL^\top$ and parameterize $L(\theta)$ to be lower-triangular, have positive diagonal, and be differentiable in the parameters $\theta$:
% to parameterize an SPD $A \in \mathbb{R}^{d\times d}$, we define $A(L) \coloneqq LL^\top$ and
%
\begin{equation}
\label{eq:L}
    \left(L(\theta)\right)_{ij}\coloneqq\begin{cases} 0 & \text{if } i<j, \\ e^{\theta_{d(d-1)/2+i}} & \text{if } i=j, \\ \theta_{(i-2)(i-1)/2+j} & \text{if } i>j, \end{cases}
\end{equation}
where $\theta\in\mathbb{R}^{d(d+1)/2}$.
% The parameterization $A(L(\theta)) = L(\theta)L(\theta)^\top$ is differentiable in $\theta$, and is SPD for any realization of $\theta \in \mathbb{R}^\frac{d(d+1)}{2}$.
% The exponent on the diagonal is intended to guarantee the uniqueness of the parameterization, as it is monotone and positive.

Both Cholesky parameterization and sequential optimization methods are well known tools.
Yet, for KF optimization \textit{from supervised data}, we are not aware of any previous attempts to apply them together, as noise estimation (\cref{algo:KF}) is typically preferred.
% Yet, to the best of our knowledge, we are the first to harness them together for optimization of KF from supervised data.

We wrap the optimization process in the Optimized KF (\textbf{OKF}) in \cref{algo:OKF}, which outputs optimized parameters $\hat{Q},\hat{R}$ for \cref{fig:KF}. % and publish it in \pypi{PyPI (link removed)}.
Note that \cref{algo:OKF} optimizes the state estimation at \textit{current} time $t$.
By switching \cref{line:update} and \cref{line:loss}, the optimization will instead be shifted to state prediction at the \textit{next} time-step (as $\hat{x}$ becomes oblivious to the current observation $z_{k,t}$). %, and \cref{algo:OKF} optimizes instead the \textit{predicted} state at the next time-step.

\textbf{Lack of theoretical guarantees:}
If \cref{assumption:KF} cannot be trusted, neither noise estimation (\cref{algo:KF}) nor OKF (\cref{algo:OKF}) can guarantee global optimality.
Still, OKF pursues the MSE objective using standard optimization tools, which achieved successful results in many non-convex problems \citep{adam_revisited} over millions of parameters~\citep{bert}.
On the other hand, noise estimation pursues a \textit{conflicting} objective, and guarantees significant sub-optimality in certain scenarios, as analyzed in \cref{sec:experiments}.

% Note that global optimality is guaranteed by neither noise estimation (\cref{algo:KF}) nor OKF (\cref{algo:OKF}) -- if \cref{assumption:KF} cannot be trusted.
% However, OKF at least addresses the desired objective -- while noise estimation suffers from goal-misalignment and significant sub-optimality, as analyzed below. %(as theoretically analyzed below), an explicit optimization at least addresses the desired objective.
% Furthermore, gradient-based optimization achieved remarkable results in many non-convex problems with local-minima~\citep{adam_revisited} and millions of parameters~\citep{bert}.

%%%%%%%%%%%%%%%%%%%%%%%%%%%%%%%%%%%%%%%%%%%%%%%%%
%%%%%%%%%%%%%%%%%%%%%%%%%%%%%%%%%%%%%%%%%%%%%%%%%

\section{OKF vs.~Neural KF: Is the Non-Linearity Helpful?}
\label{sec:nkf}

% \begin{wrapfigure}[11]{}{0.31\textwidth}
% % \begin{figure}[!t]
% \vspace{-52pt}
% \centering
% \includegraphics[width=\linewidth]{figures/res_sample_colors.png}
% \caption{\small A sample trajectory and the corresponding predictions (projected onto XY plane).
% The standard KF provides inaccurate predictions in certain turns.}
% % \vspace{-25pt}
% \label{fig:res_sample}
% % \end{figure}
% \end{wrapfigure}

In this section, we demonstrate that comparing an optimized neural network to a non-optimized baseline may lead to incorrect conclusions: the network may seem superior even if the complicated architecture has no added value.
The implied message is not against neural networks, but rather that evaluating them against a non-optimized baseline carries a crucial flaw.

% In this section, we introduce a non-linear Neural KF model (NKF) and apply it to the classic Doppler radar problem \citep{modern_radar,doppler2}.
% The NKF achieves a significantly better accuracy than the linear KF (tuned by \cref{algo:KF}).
% However, once the same KF is optimized (using \cref{algo:OKF}), the advantage of NKF is completely eliminated.
% The key message of this section is that comparing ``apples and oranges'' -- an optimized neural network to a non-optimized baseline -- can \textit{actually} lead to incorrect conclusions and to adoption of unnecessarily complex architectures.
% % over-complicated

\begin{figure}[!t]
\begin{minipage}{0.32\textwidth}
% \vspace{-52pt}
\centering
\includegraphics[width=\linewidth]{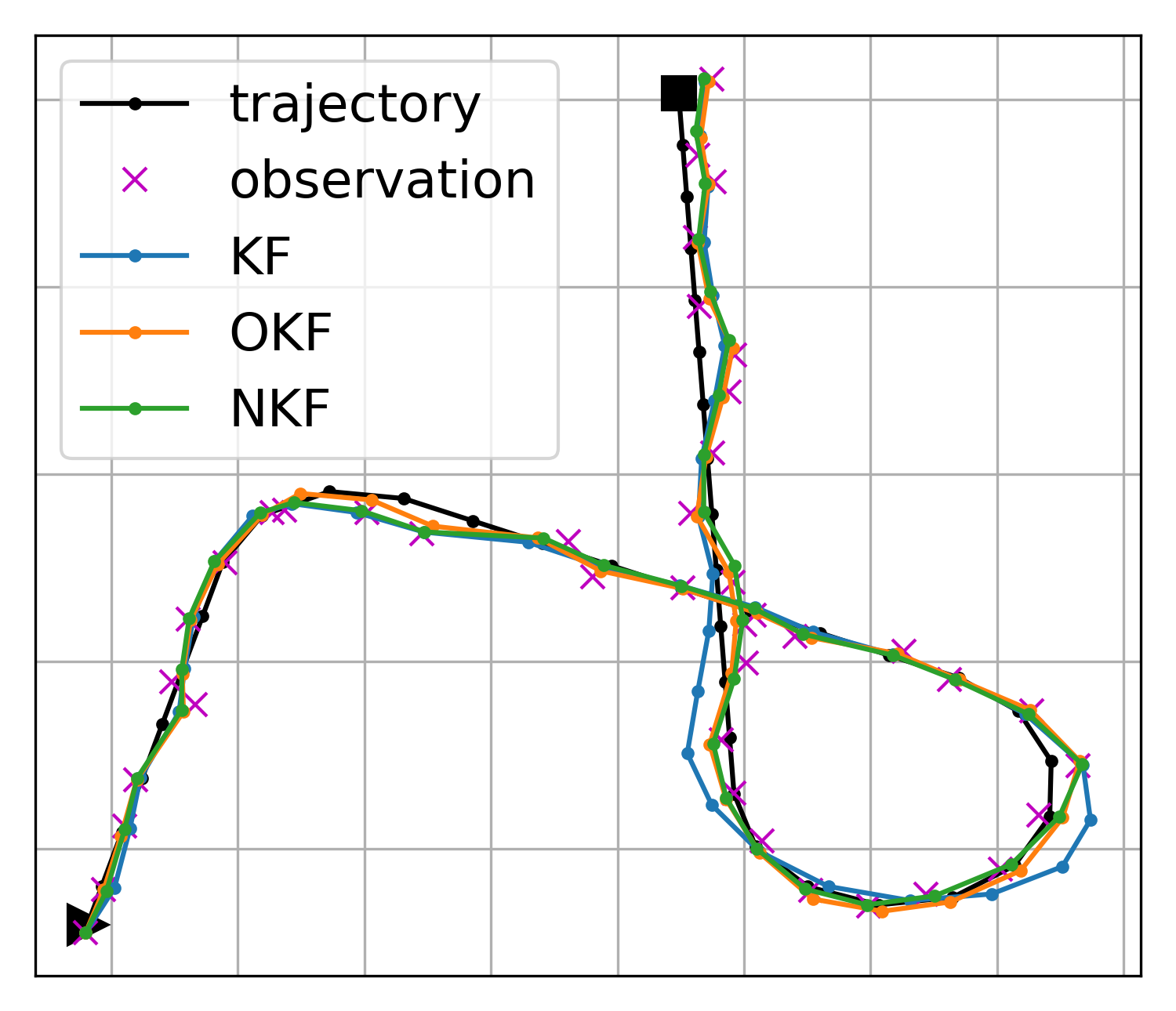}
\caption{\small A sample trajectory and the corresponding predictions (projected onto XY plane), in the Free-motion benchmark.
The standard KF provides inaccurate predictions in certain turns.}
% \vspace{-25pt}
\label{fig:res_sample}
\end{minipage}
\hspace{10pt}
\begin{minipage}{0.64\textwidth}
% \vspace{-16pt}
\centering
\begin{subfigure}{0.48\linewidth}
% \begin{subfigure}{0.29\linewidth}
  \centering
  \includegraphics[width=\linewidth]{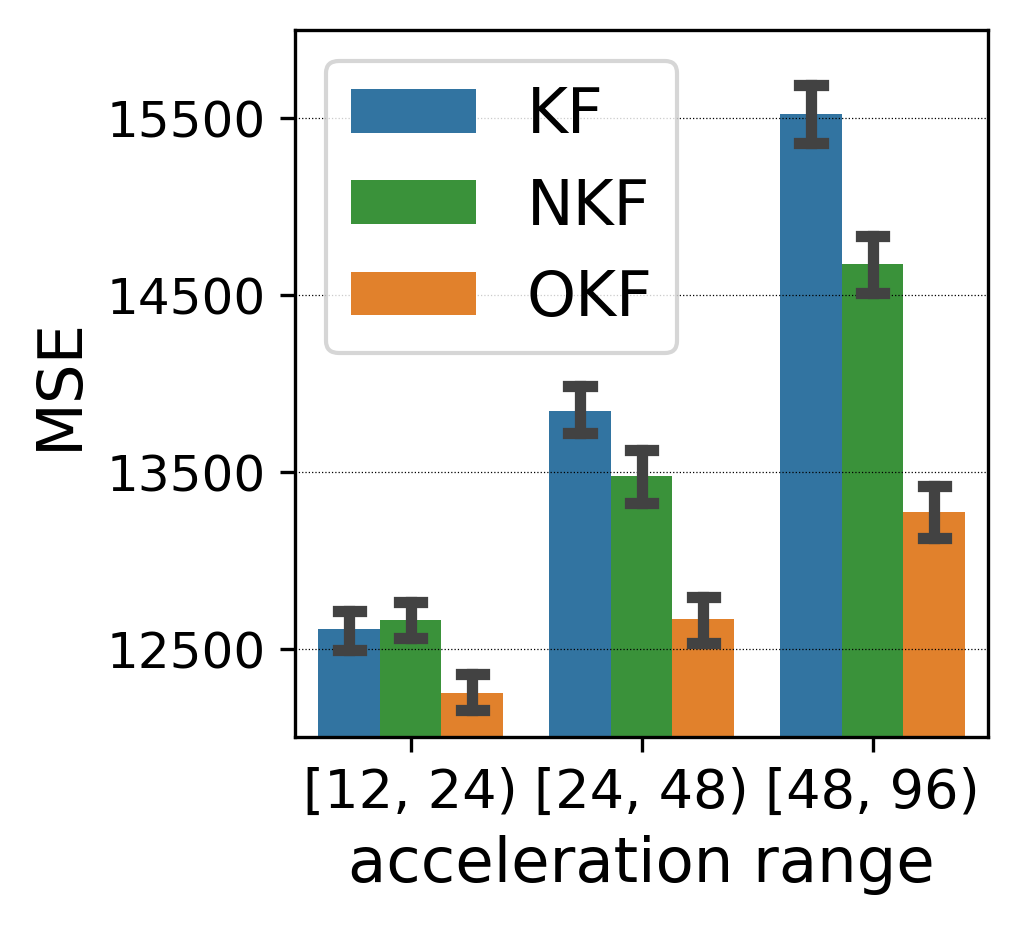}
  \caption{Free-motion benchmark}
  \label{fig:res_NKF_MSE}
\end{subfigure} %\hspace{2pt}
\begin{subfigure}{0.48\linewidth}
% \begin{subfigure}{0.29\linewidth}
  \centering
  \includegraphics[width=\linewidth]{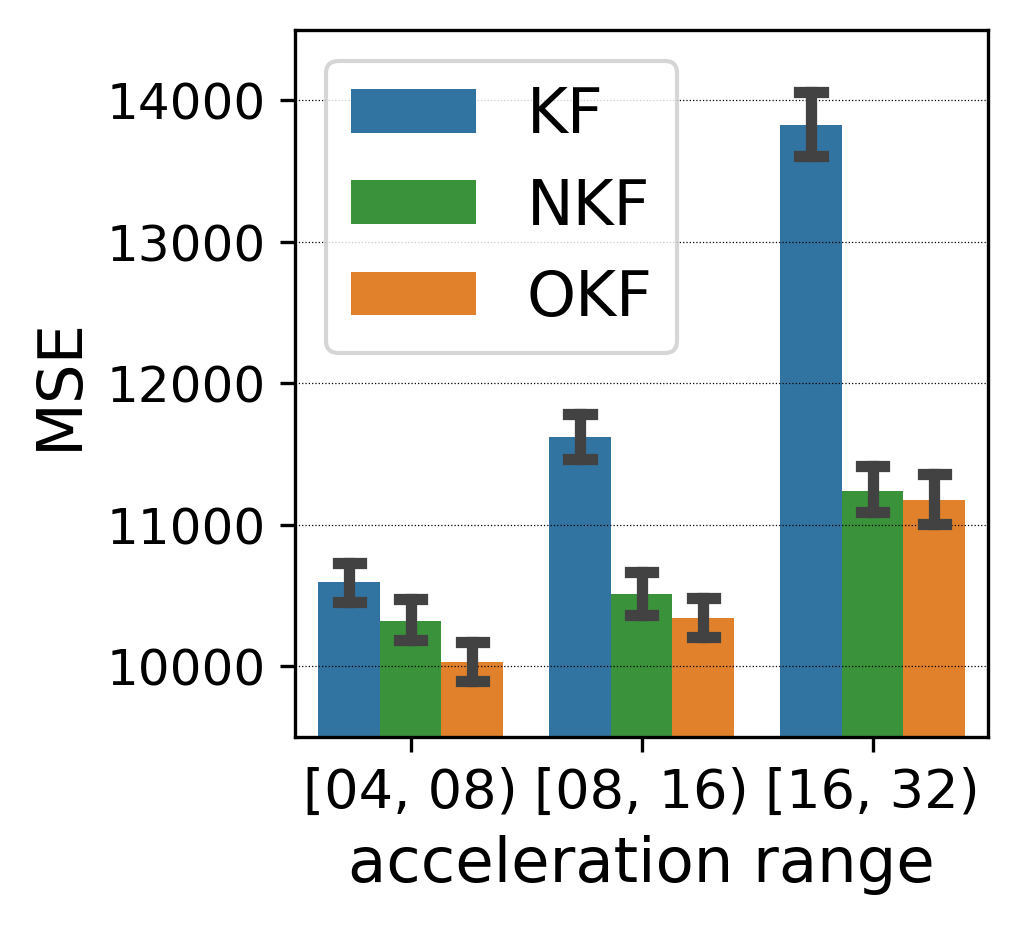}
  \caption{No-turns benchmark}
  \label{fig:res_NKF_acc_MSE}
\end{subfigure} %\hspace{5pt}
\caption{\small Test errors and 95\% confidence intervals, over targets with different accelerations.
The middle acceleration range coincides with the training accelerations (24-48 in (a) and 8-16 in (b)), and the other ranges correspond to out-of-distribution generalization.
% The error-bars correspond to 95\% confidence intervals.
}
\label{fig:res_NKF}
\end{minipage}
\end{figure}

\textbf{The Doppler radar problem:}
We consider a variant of the classic Doppler radar problem \citep{modern_radar,doppler2},
% In the Doppler radar problem,
where various targets trajectories are tracked in a homogeneous 3D space, given regular observations of a Doppler radar.
The state $X=(x_x,x_y,x_z,x_{ux},x_{uy},x_{uz})^\top\in\mathbb{R}^6$ consists of 3D location and velocity.
The goal is to minimize the MSE over the 3 location coordinates.
While the true dynamics $F$ are unknown to the KF, a constant-velocity model $\tilde{F}$ can be used:
\begin{equation}
\label{eq:doppler_F}
\tilde{F}=\left(\begin{smallmatrix} 1 &&& 1 && \\ & 1 &&& 1 & \\ && 1 &&& 1 \\&&& 1 &&\\&&&& 1 &\\&&&&& 1 \end{smallmatrix}\right) .
\end{equation}
%

% \begin{wrapfigure}[30]{}{0.31\textwidth}
% % \begin{figure}[!t]
% \vspace{-16pt}
% \centering
% \begin{subfigure}{\linewidth}
% % \begin{subfigure}{0.29\linewidth}
%   \centering
%   \includegraphics[width=\linewidth]{figures/NKF_MSE.png}
%   \caption{Free-motion benchmark}
%   \label{fig:res_NKF_MSE}
% \end{subfigure} \hspace{2pt}
% \begin{subfigure}{\linewidth}
% % \begin{subfigure}{0.29\linewidth}
%   \centering
%   \includegraphics[width=\linewidth]{figures/NKF_acc_MSE.png}
%   \caption{No-turns benchmark}
%   \label{fig:res_NKF_acc_MSE}
% \end{subfigure} \hspace{5pt}
% \caption{\small Test errors and 95\% confidence intervals, over targets with different accelerations.
% The middle acceleration range coincides with the training accelerations (24-48 in (a) and 8-16 in (b)), and the other ranges correspond to out-of-distribution generalization.
% % The error-bars correspond to 95\% confidence intervals.
% }
% % \vspace{-25pt}
% \label{fig:res_NKF}
% % \end{figure}
% \end{wrapfigure}

An observation $Z\in\mathbb{R}^4$ consists of the location in spherical coordinates (range, azimuth, elevation) and the radial velocity (the Doppler signal), with an additive i.i.d Gaussian noise. After transformation to Cartesian coordinates, the observation model can be written as:
\begin{equation}
\label{eq:doppler_H}
    H=H(X)=\left(\begin{smallmatrix} 1 \\ & 1 \\ && 1 \\ &&& \frac{x_x}{r} & \frac{x_y}{r} & \frac{x_z}{r} \end{smallmatrix}\right) ,
\end{equation}
where $r=\sqrt{x_x^2+x_y^2+x_z^2}$.
Since $H=H(X)$ relies on the unknown location $(x_x,x_y,x_z)$, we instead substitute $\tilde{H} \coloneqq H(Z)$ in the KF update step in \cref{fig:KF}.

\textbf{Neural KF:}
The Neural Kalman Filter (NKF) incorporates an LSTM model into the KF framework, as presented in \cref{app:nkf} and \cref{fig:NKF_diagram}.
% The demonstration relies on our experiments below with the Neural Kalman Filter model (NKF), which incorporates an LSTM model into the KF framework.
% , and is presented in \cref{app:nkf} and \cref{fig:NKF_diagram}.
We originally developed NKF to improve the prediction of the non-linear highly-maneuvering targets in the Doppler problem (e.g., \cref{fig:res_sample}), and made honest efforts to engineer a well-motivated architecture. % for the problem. %, as presented in \cref{app:nkf} and \cref{fig:NKF_diagram}.
Regardless, we stress that this section demonstrates a methodological flaw when comparing \textit{any} optimized filtering method to the KF; this methodological argument stands regardless of the technical quality of NKF.
In addition, \cref{app:nkf} presents similar results for other variants of NKF.

\textbf{Experiments:}
We train NKF and OKF on a dataset of simulated trajectories, representing realistic targets with free motion (as displayed in \cref{fig:res_sample}).
As a second benchmark, we also train on a dataset of simplified trajectories, with speed changes but with no turns.
The two benchmarks are specified in detail in \cref{app:okf_detailed}, and correspond to \cref{fig:trajs_consta} and \cref{fig:trajs_free}.
We tune the KF from the same datasets using \cref{algo:KF}.
In addition to out-of-sample test trajectories, we also test generalization to out-of-distribution trajectories, generated using different ranges of target accelerations (affecting both speed changes and turns radiuses).
% For testing in each benchmark, we use trajectories with an extended range of target accelerations (modifying both speed changes and turns radiuses).
% The test MSE is measured separately for these out-of-distribution accelerations.

% \cref{fig:res_NKF_MSE} summarizes the relative prediction errors in comparison to the standard KF baseline (tuned from the same data using \cref{algo:KF}).
\cref{fig:res_NKF} summarizes the test results.
Compared to KF, NKF reduces the errors in both benchmarks, suggesting that the non-linear architecture pays off.
However, optimization of the KF (using OKF) reduces the errors even further, and thus reverses the conclusion.
That is, the advantage of NKF in this problem comes \textit{exclusively} from optimization, and \textit{not at all} from the expressive architecture.
% came from optimization, and \textit{not at all} from the expressive architecture.
% Of course, neural networks may be beneficial in other scenarios; yet, in this example, without optimizing the KF, the over-complicated NKF architecture would be preferred unjustifiably.

% , another tracking benchmark
% \cref{app:nkf} extends the experiments to other variants of NKF and additional likelihood evaluation metric (NLL). % in addition to the MSE.

%%%%%%%%%%%%%%%%%%%%%%%%%%%%%%%%%%%%%%%%%%%%%%%%%
%%%%%%%%%%%%%%%%%%%%%%%%%%%%%%%%%%%%%%%%%%%%%%%%%

\section{OKF vs.~KF: } % : The Sub-Optimality of Noise Estimation}
\label{sec:experiments}

\cref{sec:nkf} presents the methodological contribution of OKF for non-linear filtering, as an optimized baseline for comparison, instead of the standard KF.
In this section, we study the advantage of OKF over the KF more generally.
We show that OKF consistently outperforms the KF in a variety of scenarios from 3 different domains.
This carries considerable practical significance: unlike neural models, shifting from KF to OKF merely requires change of the parameters $\hat{Q},\,\hat{R}$, hence can be deployed to real-world systems without additional overhead, complexity or latency on inference.

Recall that by \cref{theorem:optimality}, the KF is already optimal unless \cref{assumption:KF} is violated. Thus, the violations are discussed in depth, and the effects of certain violations are analyzed theoretically.

\subsection{Doppler Radar Tracking}
\label{sec:doppler}

\cref{theorem:optimality} guarantees the optimality of \cref{algo:KF} (KF).
Yet, in \cref{sec:nkf}, OKF outperforms the KF.
% using an identical model architecture -- just by determining the parameters differently using \cref{algo:OKF}.
This is made possible by the violation of \cref{assumption:KF}: while the Doppler problem of \cref{sec:nkf} may not seem complex, the trajectories follow a non-linear motion model (as displayed in \cref{fig:res_sample}).

% The significant sub-optimality of KF compared to OKF in \cref{sec:nkf} may come as a surprise, as the two rely on the same model architecture, and only determine the parameters differently.
% Recall that \cref{algo:KF} is the standard method for KF parameters tuning given supervised data \citep{TutorialKF,ALS}, and its optimality is guaranteed by \cref{theorem:optimality}.
% Yet, while the Doppler radar problem of \cref{sec:nkf} may not seem complex, the trajectories follow a non-linear motion model (as displayed in \cref{fig:res_sample}); hence, \cref{assumption:KF} is violated and no optimality is guaranteed.

\begin{wrapfigure}[20]{R}{0.5\textwidth} % I
\vspace{-10pt}
% \begin{figure}%[!t]
\centering
\includegraphics[width=0.8\linewidth]{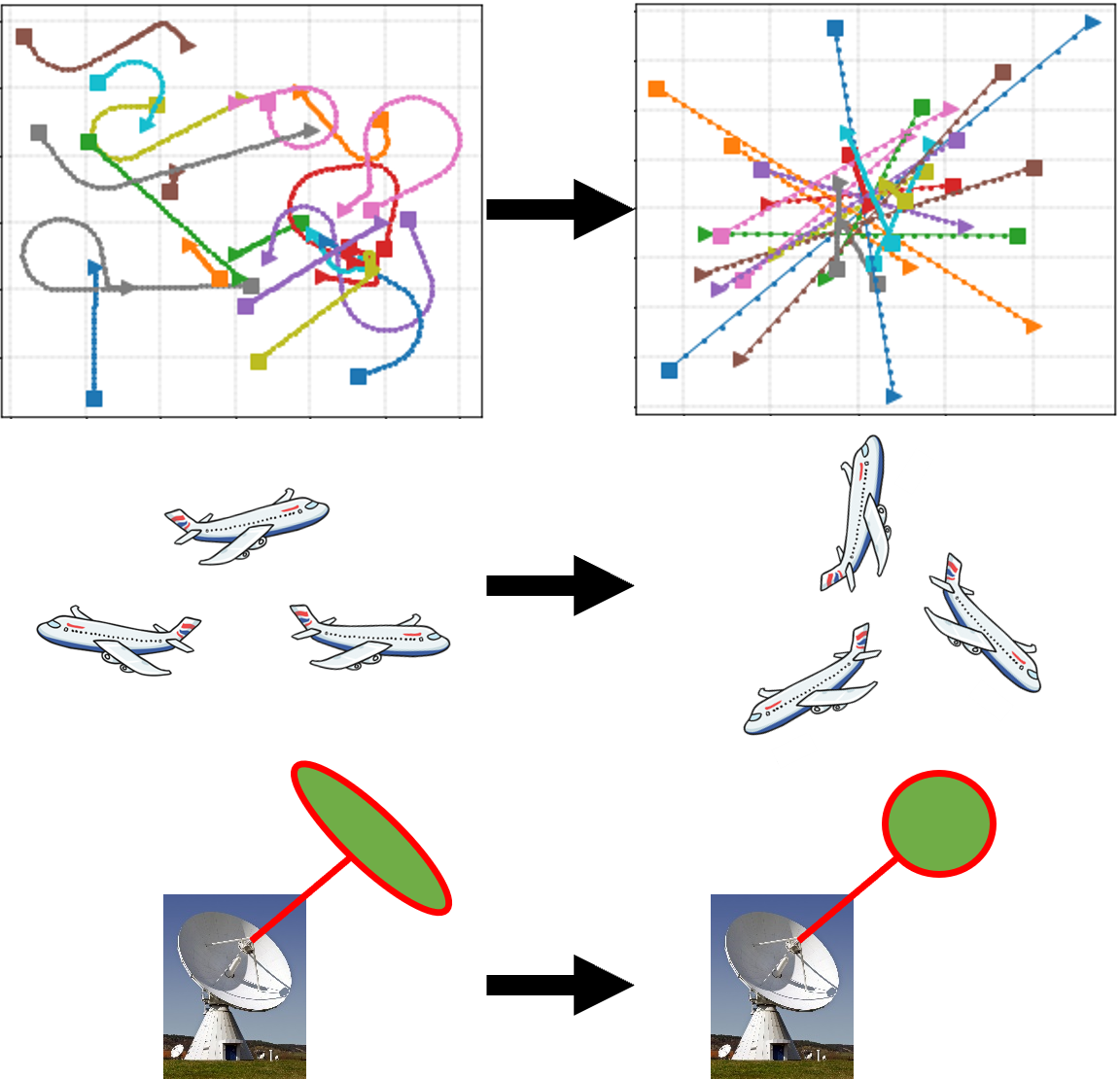}
\caption{\small The original Doppler problem (left) is simplified to a toy problem (right), with linear motion, isotropic flying directions and physically-impossible radar. After all the simplifications, \cref{assumption:KF} still does not hold, thus \cref{algo:KF} is still sub-optimal and outperformed by OKF.}
\label{fig:radar_simplification}
\end{wrapfigure}

% What would happen if we simplified the problem further?
% By simulating only constant-velocity targets, the motion model $F$ becomes linear.
% One may hope that \cref{assumption:KF} is now restored.
% However, a careful review of the assumption reveals further violations.
Imagine that we simplified the problem by only simulating constant-velocity targets, making the true motion model $F$ linear.
Would this recover \cref{assumption:KF} and make OKF unnecessary?
The answer is \textit{no}; the adventurous reader may attempt to list all the remaining violations before reading on.

The simulated targets move mostly horizontally, with limited elevation changes. This is not expressed by the KF's initial state distribution ($\hat{x}_0,\hat{P}_0$). To remedy this, one may simulate motion uniformly in all directions.
A third violation comes from the observation noise.
While the radar noise is i.i.d in spherical coordinates (as mentioned in \cref{sec:nkf}), it is not i.i.d in \textit{Cartesian} coordinates (see discussion in \cref{app:iid_violation}).
To overcome this, one may simulate a radar with (physically-impossible) Cartesian i.i.d noise.
This results in the unrealistically-simplified problem visualized in \cref{fig:radar_simplification}.

Despite the simplifications, it turns out that \cref{assumption:KF} is still not met, % Specifically,
as the observation model in \cref{eq:doppler_H} is still not linear (i.e., $H=H(X)$ is not constant).
As shown by \cref{prop:linear_violation}, this single violation alone results in a significant deviation of \cref{algo:KF} from the optimal parameters.

We first define the simplified problem.
\begin{problem}[The toy Doppler problem]
\label{problem:toy_doppler}
    The toy Doppler problem is the filtering problem modeled by \cref{eq:KF_model}, with constant-velocity dynamics $F$ (\cref{eq:doppler_F}), Doppler observation $H$ (\cref{eq:doppler_H}), and %$Q,R$ defined by
    $$ Q=\pmb{0}\in\mathbb{R}^{6\times6}, \qquad R=\left(\begin{smallmatrix} \sigma_x^2 \\ & \sigma_y^2 \\ && \sigma_z^2 \\ &&& \sigma_D^2 \end{smallmatrix}\right), $$
    where $\sigma_x,\sigma_y,\sigma_z,\sigma_D>0$.
\end{problem}
Recall that $H=H(X)$ in \cref{eq:doppler_H} depends on the state $X$, which is unknown to the model.
Thus, we assume that $\tilde{H}=H(\tilde{X})$ is used in the KF update step (\cref{fig:KF}), with some estimator $\tilde{X}\approx X$ (e.g., $\tilde{H}=H(Z)$ in \cref{sec:nkf}).
Hence, the \textit{effective} noise is $\tilde{R} \coloneqq Cov(Z-\tilde{H}X) \ne Cov(Z-HX) = R$.
\cref{prop:linear_violation} analyzes the difference between $\tilde{R}$ and $R$.
% the measurement noise ${R} = Cov(Z-{H}X)$ and the \textit{effective} noise $\tilde{R} \coloneqq Cov(Z-\tilde{H}X)$.
% To simplify \cref{prop:linear_violation},
To simplify the analysis, we further assume that the error $\tilde{X}-X$ within $\tilde{H}$ (e.g., $Z-X$) is independent of the target velocity.
\begin{proposition}
    \label{prop:linear_violation}
    In the toy Doppler \cref{problem:toy_doppler} with the estimated observation model $\tilde{H}$, the effective observation noise $\tilde{R} = Cov(Z-\tilde{H}X)$ is:
    % Then, the observation noise corresponding to $\tilde{H}$ (the \textit{effective} observation noise) is described by the following covariance matrix:
    \begin{equation}
    \label{eq:R}
        \tilde{R} = \left(\begin{smallmatrix} \sigma_x^2 \\ & \sigma_y^2 \\ && \sigma_z^2 \\ &&& \sigma_D^2 + C \end{smallmatrix}\right) = R + \left(\begin{smallmatrix} 0 \\ & 0 \\ && 0 \\ &&& C \end{smallmatrix}\right) ,
    \end{equation}
    % for some $C>0$.
    where $C = \Omega(\mathbb{E}[||u||^2])$ is the asymptotic lower bound (``big omega'') of the expected square velocity $u$. In particular, $C>0$ and is unbounded as the typical velocity grows.
\end{proposition}
\begin{proof}[Proof sketch (see complete proof in \cref{app:linear_violation})]
We have $Cov(Z-\tilde{H}X) = Cov(Z-HX + (H-\tilde{H})X) = R + Cov((H-\tilde{H})X)$,
where the last equality relies on the independence between the target velocity and the estimation error $\tilde{X}-X$.
We then calculate $Cov((H-\tilde{H})X)$.
\end{proof}

\begin{wrapfigure}[18]{}{0.47\textwidth}
% \begin{figure}%[!h]
\vspace{-15pt}
\centering
\begin{subfigure}{.45\linewidth}
  \centering
  \includegraphics[width=1.\linewidth]{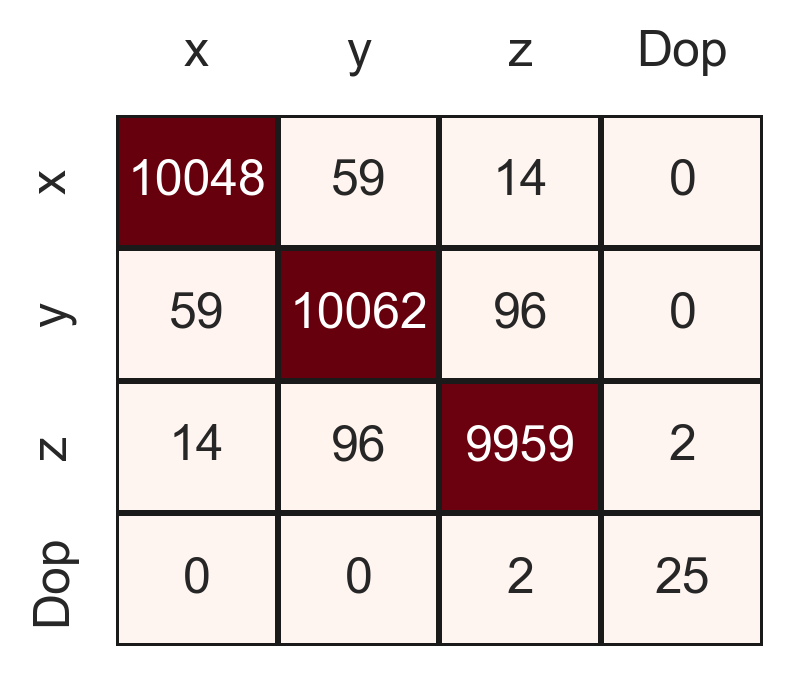}
  \caption{\small KF / $\hat{R}$}
  \label{fig:noise_R_KF_toy}
\end{subfigure}
\begin{subfigure}{.45\linewidth}
  \centering
  \includegraphics[width=1.\linewidth]{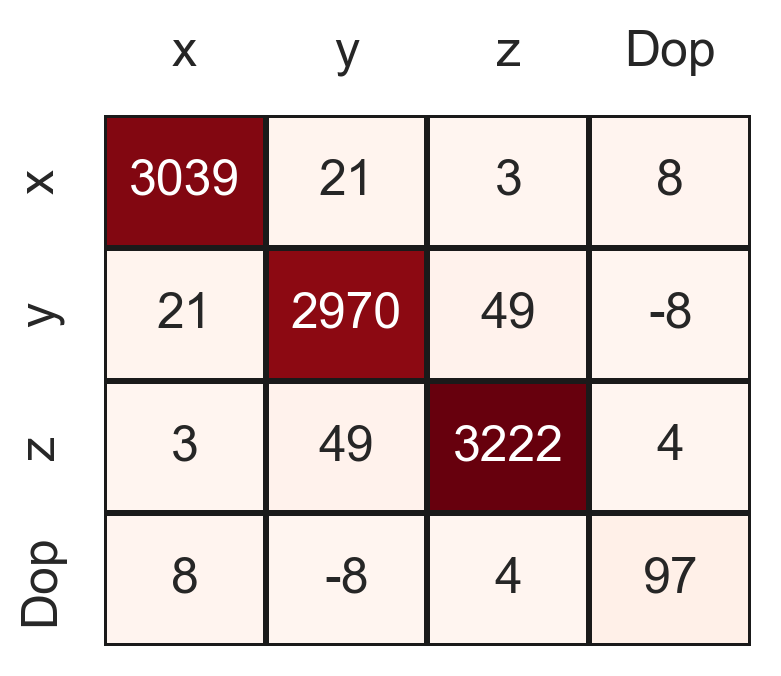}
  \caption{\small OKF / $\hat{R}$}
  \label{fig:noise_R_OKF_toy}
\end{subfigure}
\caption{\small The parameters $\hat{R}$ learned by KF and OKF in the toy Doppler problem.
The rows and columns' entries correspond to location ($x,y,z$) and radial velocity ($Doppler$).
The simulated noise variance is $100^2$ for the positional dimensions and $5^2$ for velocity, and is estimated accurately by the KF.
However, OKF increases the noise associated with velocity, in accordance with \cref{prop:linear_violation}.
The decrease in the positional variance comes from scale-invariance in the toy problem, as discussed in \cref{app:linear_violation}.
}
\label{fig:noise_params_toy}
% \end{figure}
\end{wrapfigure}

\cref{prop:linear_violation} has an intuitive interpretation: when measuring the velocity, \cref{algo:KF} only considers the inherent Doppler signal noise $\sigma_D$. However, the \textit{effective} noise $\sigma_D+C$ also includes the \textit{transformation error} from Doppler to the Cartesian coordinates, caused by the uncertainty in $H(X)$ itself.
Notice that heuristic solutions such as inflation of $R$ would not recover the effective noise $\tilde{R}$, which only differs from $R$ in one specific entry.
% While \cref{algo:KF} would estimate $R$ from supervised data, according to \cref{prop:linear_violation} the optimal noise parameter is $\tilde{R}$.

Yet, as demonstrated below, OKF captures the effective noise $\tilde{R}$ successfully.
Critically, it does so from mere data: OKF does not require the user to specify the model correctly, or to even be aware of the violation of \cref{assumption:KF}.

\textbf{Experiments:}
We test KF and OKF on the toy \cref{problem:toy_doppler} using the same methodology as in \cref{sec:nkf}.
In accordance with \cref{prop:linear_violation}, OKF adapts the Doppler noise parameter:
as shown in \cref{fig:noise_params_toy}, it increases $\sigma_D$ in proportion to the location noise by a factor of $\approx 13$.
Note that we refer to the proportion instead of absolute values due to scale-invariance in the toy problem, as discussed in \cref{app:linear_violation}.
% As shown in \cref{fig:noise_params_toy}, the ratio between the two is $\approx 13$ times larger in OKF in comparison to KF (we refer to the ratio and not to the absolute values due to scale-invariance of the toy problem, as discussed in \cref{app:linear_violation}).
Following the optimization, \textbf{OKF reduces the test MSE by 44\%} -- from 152 to 84.
% As displayed in \cref{fig:noise_params_toy}, and , OKF learns to increase the Doppler noise parameter in comparison to the location noise. %positional noise.

% In this toy problem, the optimal parameters could in fact be derived analytically from \cref{prop:linear_violation}.
% In practical problems, however, analytical solution is often infeasible.
% In fact, as discussed above, even specifying the model itself is not always trivial.
% Clearly, analytical solution of the wrong model would result in unaware sub-optimality.
% Instead, OKF optimizes the prediction errors directly from data, without any prior knowledge of the model.

\textbf{Extended experiments:}
% Sections \ref{sec:nkf} and \ref{sec:doppler}
% Sections \ref{sec:nkf} and \ref{sec:doppler}
This section and \cref{sec:nkf} test OKF against KF in three specific variants of the Doppler problem.
One may wonder if OKF's advantage generalizes to other scenarios, such as:
\begin{itemize}
    \vspace{-5pt}
    \item Different subsets of violations of \cref{assumption:KF}; 
    % \vspace{-5pt}
    \item Other baseline models than KF, e.g., Optimized Extended-KF;
    %\item An oracle KF baseline that uses the true noise covariances;
    % \vspace{-5pt}
    \item Small training datasets;
    % \vspace{-5pt}
    \item Generalization to out-of-distribution test data.
\end{itemize} \vspace{-5pt}

\begin{wrapfigure}[9]{}{0.32\textwidth}
% \begin{figure}[!b]
\vspace{-46pt}
\centering
\includegraphics[width=\linewidth]{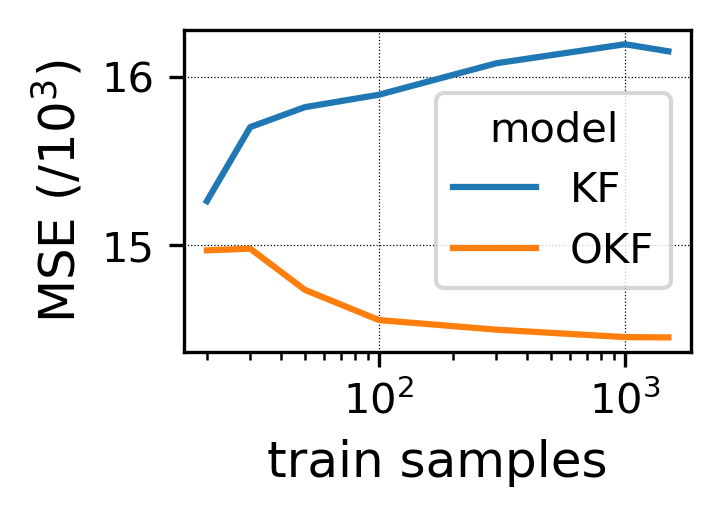}
\caption{\small Different train data sizes in the Doppler problem (\cref{fig:res_sample}). Due to objective misalignment, \cref{algo:KF} deteriorates with more train data.}
\label{fig:train_size_front}
% \end{figure}
\end{wrapfigure}

The extended experiments in \cref{app:okf} address \textit{all} of the concerns above by examining a wide range of problem variations in the Doppler radar domain.
In addition, other domains are experimented below.
\textbf{In all of these experiments, OKF outperforms \cref{algo:KF} in terms of MSE}.
% Interestingly, \cref{algo:KF} is also shown to sometimes deteriorate with the amount of train data.

Finally, the goal-misalignment of \cref{algo:KF} is demonstrated directly by two results:
even oracle noise estimation fails to optimize the MSE (\cref{app:okf_detailed}); and feeding more data to \cref{algo:KF} may \textit{degrade} the MSE (\cref{app:okf_train_size}). \cref{fig:train_size_front} presents a sample of the results of \cref{app:okf_train_size}.

%%%%%%%%%%%%%%%%%%%%%%%%%%%%%%%%%%%%%%%%%%%%%%%%%

\newpage
\subsection{Video Tracking}
\label{sec:video}

% \begin{figure}
% % \vspace{-37pt}
% \centering
% \includegraphics[width=.6\linewidth]{figures/MOT20/MOT20_sample.jpg}
% \caption{\small A sample of 2 trajectories in the first frame of MOT20 test video, along with the predictions of KF and OKF.}
% \label{fig:MOT20_sample}
% \end{figure}

\begin{figure}%[!b]
\begin{minipage}{\textwidth}
% \begin{wrapfigure}[14]{}{0.57\textwidth}
% \begin{figure}
% \vspace{-37pt}
\centering
\includegraphics[width=.6\linewidth]{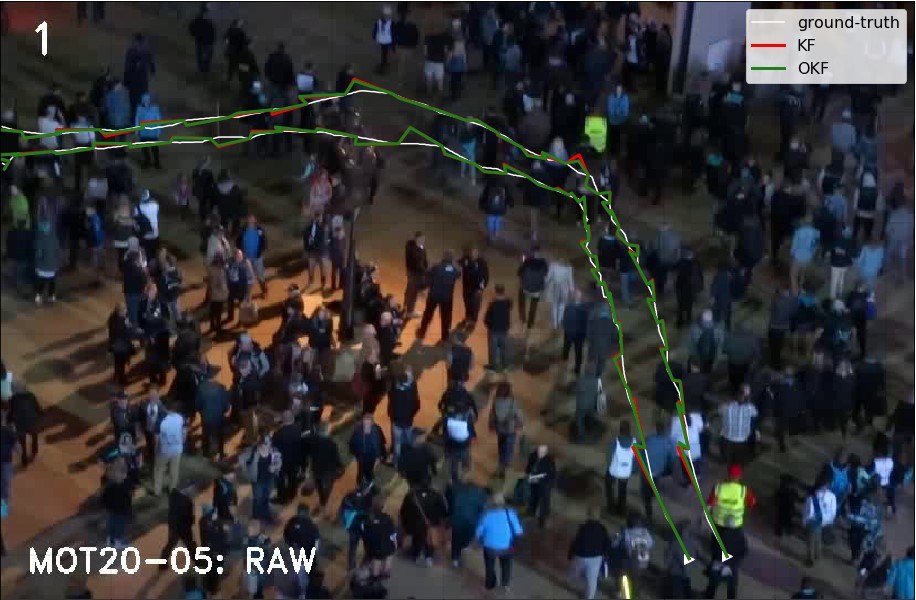}
\caption{\small A sample of 2 trajectories in the first frame of MOT20 test video, along with the predictions of KF and OKF.}
\label{fig:MOT20_sample}
\end{minipage}
\begin{minipage}{\textwidth}
% \vspace{-16pt}
\centering
\begin{subfigure}{0.3\linewidth}
    \centering
    \includegraphics[width=\linewidth]{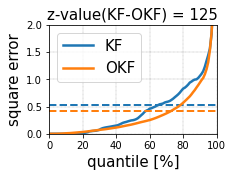}
    \caption{Video tracking}
    \label{fig:MOT20_res}
\end{subfigure}
\begin{subfigure}{0.3\linewidth}
    \centering
    \includegraphics[width=\linewidth]{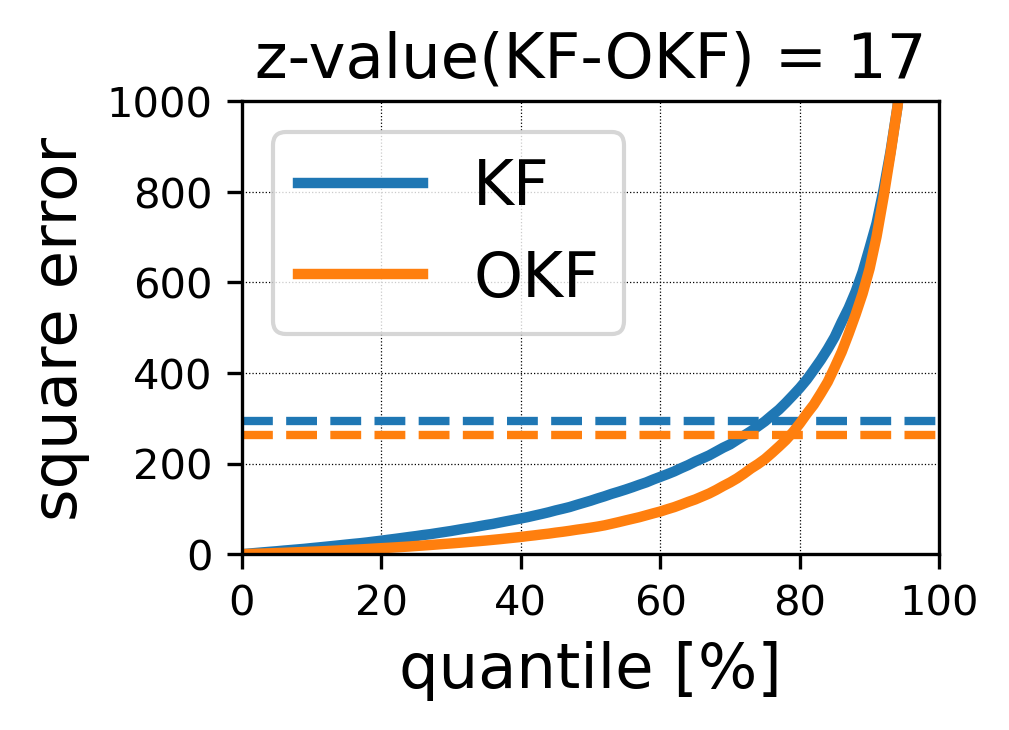}
    \caption{Lidar-based state estimation}
    \label{fig:lidar_MSE}
\end{subfigure}
\caption{\small Summary of the test errors in the video and lidar problems. The dashed lines correspond to MSE.
Both z-values correspond to $\text{p-value}<10^{-6}$.
Each z-value is calculated over $N$ test trajectories as follows: $z = \frac{\mathrm{mean}(\{\Delta_i\})}{\mathrm{std}(\{\Delta_i\})} \sqrt{N}$, where $\Delta_i = err_i(KF)^2 - err_i(OKF)^2$ is the square-error difference on trajectory $1\le i\le N$.}
\label{fig:non_doppler_res}
\end{minipage}
\end{figure}

The MOT20 dataset~\citep{MOT20} contains videos of real-world targets (mostly pedestrians, as shown in \cref{fig:MOT20_sample}), along with their true location and size in every frame.
For our experimental setup, since object detection is out of the scope, we assume that the true locations are known in real-time.
The objective is to predict of the target location in the next frame.
The state space corresponds to the 2D location, size and velocity, and the observations include only the location and size.
The underlying dynamics $F$ of the pedestrians are naturally unknown, and the standard constant-velocity model is used for $\tilde{F}$.
This results in the following model:
$$ \tilde{F}=\left(\begin{smallmatrix} 1&&&&1& \\ &1&&&&1 \\ &&1&&& \\ &&&1&& \\ &&&&1& \\ &&&&&1 \end{smallmatrix}\right), \quad \tilde{H}=H=\left(\begin{smallmatrix} 1&&&&0&0 \\ &1&&&0&0 \\ &&1&&0&0 \\ &&&1&0&0 \end{smallmatrix}\right) . $$
Notice that the known observation model $\tilde{H}=H$ is \textit{linear} ($H$ is independent of $X$), hence poses a substantial difference from \cref{sec:doppler} in terms of violations of \cref{assumption:KF}.

The first three videos with 1117 trajectories are used for training, and the last video with 1208 trajectories for testing.
% We use different videos for training and testing.
As shown in \cref{fig:MOT20_res}, OKF reduces the test MSE by 18\% with high statistical significance.
% The significance level is indicated by the z-value, calculated over the test error differences $\{\Delta err_i\}_{i=1}^N$ between KF and OKF: $z = \sqrt{N}\frac{mean(\Delta err)}{std(\Delta err)}$.
% We have $z = 125$, which corresponds to $\text{p-value}<10^{-6}$.

% % \begin{wrapfigure}[17]{}{0.42\textwidth} % 20, I
% \begin{figure}
% %   \vspace{-10pt}
%   \begin{center}
%   \includegraphics[width=0.7\linewidth]{figures/MOT20/MOT20_SE.png}
%   \end{center}
% \caption{\small Prediction errors of KF and OKF on 1208 targets in the test video of MOT20. The MSE of OKF is lower by 18\%.}
% \label{fig:MOT20_res}
% \end{figure}
% % \end{wrapfigure}

%%%%%%%%%%%%%%%%%%%%%%%%%%%%%%%%%%%%%%%%%%%%%%%%%

\subsection{Lidar-based State Estimation in Self Driving}
\label{sec:lidar}

Consider the problem of state-estimation in self-driving, based on lidar measurements with respect to known landmarks~\citep{lidar_from_beacons}.
The objective is to estimate the current vehicle location.
We assume a single landmark (since the landmark matching problem is out of scope).
% The targets are generated with random trajectories as demonstrated in \cref{fig:lidar_sample}, and described in \cref{app:lidar}.
We simulate driving trajectories consisting of multiple segments, with different accelerations and turn radiuses (see \cref{fig:lidar_trajectories} in the appendix).
% The dynamics are simulated as follows: each target trajectory consists of several intervals, in each one the target has certain (positive or negative) acceleration and certain lateral acceleration ("turn magnitude"), both drawn randomly at the beginning of the interval.
The state is the vehicle's 2D location and velocity, and $\tilde{F}$ is modeled according to constant-velocity.
The observation (both true $H$ and modeled $\tilde{H}$) corresponds to the location, with an additive Gaussian i.i.d noise in polar coordinates.
This results in the following model:
$$ \tilde{F}=\left(\begin{smallmatrix} 1&0&1&0 \\ 0&1&0&1 \\ 0&0&1&0 \\ 0&0&0&1 \end{smallmatrix}\right), \quad \tilde{H}=H=\left(\begin{smallmatrix} 1&0&0&0 \\ 0&1&0&0 \end{smallmatrix}\right) . $$
We train KF and OKF over 1400 trajectories and test them on 600 trajectories.
As shown in \cref{fig:lidar_MSE}, OKF reduces the test MSE by 10\% with high statistical significance. %($\text{z-value} = 17$).

Notice that the lidar problem differs from \cref{sec:doppler} in the linear observation model $H$, and from \cref{sec:video} in the additive noise.
Both properties have a major impact on the problem, as analyzed in \cref{prop:linear_violation} and below, respectively.

% In addition, Section~\ref{sec:theory_noniid} provides theoretical analysis for the optimal parameters under the coordinates mismatch between the state and the source of the noise.

\textbf{Theoretical analysis:}
As mentioned in \cref{sec:doppler} and discussed in \cref{app:iid_violation}, the i.i.d noise in polar coordinates is not i.i.d in Cartesian ones.
To isolate the i.i.d violation and study its effect, we define a simplified toy model -- with simplified states, no-motion model $F$, isotropic motion noise $Q$ and only radial observation noise.
Note that in contrast to \cref{sec:doppler}, the observation model is already linear.
% We seize the opportunity to isolate the i.i.d violation and study its effect.
% First, we define a simplified toy model -- with simplified states, no-motion model $F$, isotropic motion noise $Q$ and only radial observation noise.
% % without velocity in the state, with a no-motion model $F$, with isotropic motion noise $Q$, and with only radial observation noise.
%
\begin{problem}[The toy lidar problem]
\label{problem:toy_lidar}
    The toy lidar problem is the filtering problem modeled by \cref{eq:KF_model} with the following parameters:
    $$F=H=\begin{pmatrix} 1 & 0 \\ 0 & 1 \end{pmatrix}, \  Q=\begin{pmatrix} q & 0 \\ 0 & q \end{pmatrix}, \  R_{polar}=\begin{pmatrix} r_0 & 0 \\ 0 & 0 \end{pmatrix} , $$
    for some unknown $q,r_0>0$, with observation noise drawn i.i.d from $\mathcal{N}(0,R_{polar})$ in \textit{polar} coordinates.
    The initial state $X_0$ follows a radial distribution (i.e., with a PDF of the form $f(||x_0||)$).
\end{problem}
\begin{proposition}
\label{prop:iid_violation}
As the number $N$ of train trajectories in \cref{problem:toy_lidar} grows, the noise parameter $\hat{R}_N(KF)$ estimated by \cref{algo:KF} converges almost surely:
$$ \hat{R}_N(KF) \xrightarrow{\text{a.s.}} \hat{R}_{est} = \begin{pmatrix} r_0/2 & 0 \\ 0 & r_0/2 \end{pmatrix} . $$
On the other hand, under regularity assumptions, the MSE is minimized by the parameter $\hat{R}_{opt} = \left(\begin{smallmatrix} r & 0 \\ 0 & r \end{smallmatrix}\right)$, where $r<{r_0}/{2}$.
\end{proposition}
\begin{proof}[Proof sketch (see complete proof in \cref{app:iid_violation})]
For $\hat{R}_{est}$, we calculate $\mathbb{E}[\hat{R}_N(KF)]$ and use the law of large numbers.
For the calculation, we transform $R_{polar}$ to Cartesian coordinates using the random direction variable $\theta$, and take the expectation over $\theta\sim U([0,2\pi))$. The uniform distribution of $\theta$ comes from the radial symmetry of the problem.
For $\hat{R}_{opt}$, we calculate and minimize the expected square error directly.
\end{proof}

\begin{figure}[!b]
% \begin{wrapfigure}[6]{R}{0.3\textwidth}
% \begin{figure}%[!h]
% \vspace{-62pt}
\centering
\begin{subfigure}{.16\linewidth}
  \centering
  \includegraphics[width=1.\linewidth]{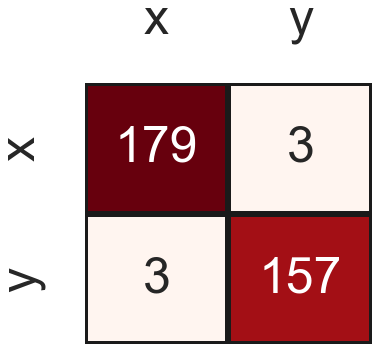}
  \caption{\small KF / $\hat{R}$}
  \label{fig:lidar_noise_R_KF}
\end{subfigure}
\begin{subfigure}{.16\linewidth}
  \centering
  \includegraphics[width=1.\linewidth]{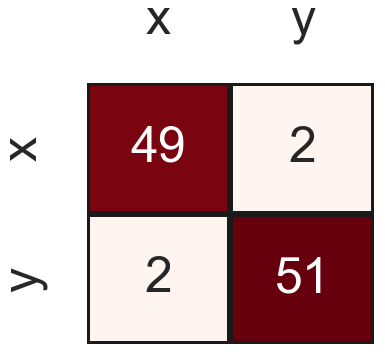}
  \caption{\small OKF / $\hat{R}$}
  \label{fig:lidar_noise_R_OKF}
\end{subfigure}
\caption{\small The parameters $\hat{R}$ learned in the lidar problem. From data alone, OKF learns to decrease the noise parameters, consistently with \cref{prop:iid_violation}.}
\label{fig:lidar_noise}
% \end{wrapfigure}
\end{figure}

Intuitively, \cref{prop:iid_violation} shows that the i.i.d violation reduces the \textit{effective} noise.
Note that the analysis only holds for the unrealistic toy \cref{problem:toy_lidar}.
% in the toy \cref{problem:toy_lidar}, the optimal parameters $R$ are lower than the ones obtained by noise estimation.
The empirical setting in this section is less simplistic, and generalizing \cref{prop:iid_violation} is not trivial.
Fortunately, OKF optimizes directly from the data, and does not require such theoretical analysis.
\cref{fig:lidar_noise} shows that indeed, in accordance with the intuition of \cref{prop:iid_violation}, OKF learns to reduce the values of $\hat{R}$ in comparison to KF.
This results in reduced test errors as specified above.

%%%%%%%%%%%%%%%%%%%%%%%%%%%%%%%%%%%%%%%%%%%%%%%%%
%%%%%%%%%%%%%%%%%%%%%%%%%%%%%%%%%%%%%%%%%%%%%%%%%

\section{Related Work}
\label{sec:related_work}

\textbf{Noise estimation}
% Estimation of the KF noise parameters has been studied for decades.
% When tuning a KF, ground-truth states data is often unavailable~\citep{noise_cov_estimation}.
of the KF parameters from observations alone has been studied for decades, as supervised data (\cref{def:supervised_data}) is often unavailable.
Various methods were studied, based on autocorrelation~\citep{Mehra,Carew}, EM~\citep{EM_for_noise_estimation} and others~\citep{ALS,seq_cov_estimation,measurement_noise_recommendation}.
When supervised data \textit{is} available, noise estimation reduces to \cref{eq:noise_estimation} and is considered a solved problem~\citep{ALS}.
We show that while noise estimation is indeed easy from supervised data, it may be the wrong objective to pursue.
% In this work we assume that the ground-truth is available, where the current standard approach is to simply estimate the noise covariance matrix from the data.
%, which is often possible in simulations and controlled experiments, such as the concrete Doppler radar use-case that motivated this work in the first place.

% Many works addressed the problem of non-stationary noise estimation~\citep{cov_estimation_varying_processes,adaptive_noise_covariance}.
% However, as demonstrated in \cref{sec:nkf}, stationary methods may be highly competitive if tuned correctly -- even in problems with complicated dynamics.

\textbf{Optimization:}
We apply gradient-based optimization to the KF with respect to its errors.
In absence of supervised data, gradient-based optimization was suggested for other losses, such as smoothness \citep{BoydOptimization}.
In the supervised setting, noise estimation is typically preferred \citep{ALS}, although optimization without gradients was suggested in \citet{Thrun_trainingKF}.
% Optimization without gradients was already suggested in \citet{Thrun_trainingKF}.
% Optimization of KF with respect to its errors was already suggested in \citet{Thrun_trainingKF}, using a method that \textit{avoids} gradients computation.
In practice, ``optimization'' of KF is sometimes handled by trial and error \citep{learning_in_indoor_navigation} or grid search \citep{noise_cov_estimation,pose_estimation}.
In other cases, $Q$ and $R$ are restricted to be diagonal \citep{unsupervised_KF_tuning,noise_cov_estimation}.
However, such heuristics may not suffice when the optimal parameters take a non-trivial form. %(such as their form in \cref{prop:linear_violation}).
% However, the optimal parameters may take a non-trivial form (as demonstrated in \cref{prop:linear_violation}), which would be hard to find via such heuristics.
% As demonstrated in \cref{prop:linear_violation}, when \cref{assumption:KF} is violated, the optimal parameters may deviate from the noise parameters in a non-trivial way, and not just by a scalar multiplication; this might be missed by naive methods such as trial-and-error.

% Not central point, and already discussed once.
% Gradient-based optimization of SPD matrices in general was suggested in \citet{matrix_exponentiated_gradient_updates} using matrix-exponents, and is also possible using projected gradient-descent~\citep{projected_GD} -- both rely on SVD-decomposition.
% In this work, we apply gradient-based optimization using the parameterization that was suggested in \citet{cov_parameterization}, which requires a mere matrix multiplication, and thus is both efficient and easy to implement.

% In absence of a trajectories dataset, a recent line of works~\citep{tsiamis2020online, goel2021regret} minimizes a regret metric, using online optimization from the observations of the current trajectory.

\textbf{Neural Networks (NNs) in filtering:}
The NKF in \cref{sec:nkf} relies on a recurrent NN.
NNs are widely used in non-linear filtering, e.g., for online prediction \citep{ManeuveringTargetTracking2,KF_vs_LSTM,pose_estimation,navigation_using_RNN,NESDE}, near-online prediction \citep{vision_tracking_with_bilinear_LSTM}, and offline prediction \citep{DeepMTT}.
Learning visual features for tracking via a NN was suggested by \citet{deep_sort}.
% In addition, while \citet{sort} apply a KF for video tracking from mere object detections, \citet{deep_sort} add to the same system a NN that generates visual features as well.
NNs were also considered for related problems such as data association \citep{DeepDA}, model switching \citep{IMM_with_RNN_transitions}, and sensors fusion \citep{sensor_fusion}.

In addition, all the 10 studies cited in \cref{sec:intro} used a NN model for non-linear filtering, with either KF or EKF as a baseline for comparison.
As discussed above, none has optimized the baseline model to a similar extent as the NN. As demonstrated in \cref{sec:nkf}, such methodology could lead to unjustified conclusions. %, resulting in ``apples and oranges'' comparison.

% Repeating.
% Many works that consider NNs for filtering problems, use a KF as a baseline for comparison.
% However, while the NN parameters are typically optimized with respect to the prediction errors, the KF parameters tuning is sometimes ignored \citep{ManeuveringTargetTracking2,bai2020neuron,RNN_EKF}, sometimes based on estimation (or knowledge) of the noise \citep{navigation_using_RNN,ANN_vs_EKF,KalmanNet}, and sometimes optimized heuristically as mentioned above \citep{learning_in_indoor_navigation,pose_estimation,KF_with_ANN}.
% Our findings imply that this methodology is wrong, since the baseline is not optimized to the same level as the learning model.
% \citet{KF_vs_NN_for_batteries} explicitly discusses the sensitivity of EKF to the noise model accuracy, and suggests the solution of a NN with supervised learning -- without ever considering the same supervised learning for the EKF itself.

%%%%%%%%%%%%%%%%%%%%%%%%%%%%%%%%%%%%%%%%%%%%%%%%%
%%%%%%%%%%%%%%%%%%%%%%%%%%%%%%%%%%%%%%%%%%%%%%%%%

\section{Summary}
\label{sec:summary}

We observed that violation of the KF assumptions is common, and is potentially difficult to notice or model.
Under such violation, we analyzed (theoretically and empirically) that the standard noise estimation of the KF parameters conflicts with MSE optimization.
An immediate consequence is that the KF is often used sub-optimally.
A second consequence is that in many works in the literature, where a neural network is compared to the KF, the experiments become inconclusive: they cannot decide whether the network succeeded due to superior architecture, or merely because its parameters were optimized.
We presented the Optimized KF (OKF), and demonstrated that it can solve both issues (\cref{sec:experiments} and \cref{sec:nkf}, respectively).

From a practical point of view, the OKF is available on \pypi{PyPI} and is easily applicable to new problems. Since its architecture is identical to the KF, and only the parameters are changed, the learned model causes neither inference-time delays nor deployment overhead.
All these properties make the OKF a powerful practical tool for both linear and non-linear filtering problems.

\subsubsection*{Acknowledgements}
The authors thank Tal Malinovich, Ophir Nabati, Zahar Chikishev, Mark Kozdoba, Eli Meirom, Elad Sharony, Itai Shufaro and Shirli Di-Castro for their helpful advice.
This work was partially funded by the European Union's Horizon Europe Programme, under grant number 101070568.

% \subsubsection*{Reproducibility}
% All the experiments in this work are reproducible using our
% \href{https://anonymous.4open.science/r/OKF_anonymous_tmp-18C0}{\underline{code}}, including data generation, models training and results analysis.
% The complete proofs for the theoretical results are available in Appendices \ref{sec:toy_analysis} and \ref{sec:theory_noniid}.

\newpage
\bibliographystyle{plainnat}
% \nocite{*}
\bibliography{main}

\newpage
\appendix

%%%%%%%%%%%%%%%%%%%%%%%%%%%%%%%%%%%%%%%%%%%%%%%%%
%%%%%%%%%%%%%%%%%%%%%%%%%%%%%%%%%%%%%%%%%%%%%%%%%

% \section*{Table of Contents}
\setcounter{tocdepth}{1}
\tableofcontents
\newpage

%%%%%%%%%%%%%%%%%%%%%%%%%%%%%%%%%%%%%%%%%%%%%%%%%
%%%%%%%%%%%%%%%%%%%%%%%%%%%%%%%%%%%%%%%%%%%%%%%%%
% \section{Detailed Problem Setups}
% \label{app:setups}
%%%%%%%%%%%%%%%%%%%%%%%%%%%%%%%%%%%%%%%%%%%%%%%%%
% \subsection{Doppler Radar Tracking}
% \label{app:doppler}
%%%%%%%%%%%%%%%%%%%%%%%%%%%%%%%%%%%%%%%%%%%%%%%%%
% \subsection{Video Tracking}
% \label{app:video}
% For the train data we use the videos MOT20-01,MOT20-02,MOT20-03, and for the test data MOT20-05. In particular, our test data comes from an entirely different video than the train data, hence the testing is less prone to overfit.
%%%%%%%%%%%%%%%%%%%%%%%%%%%%%%%%%%%%%%%%%%%%%%%%%
% \subsection{Lidar-based State Estimation in Self Driving}
% \label{app:lidar}
% The dynamics are simulated as follows: each target trajectory consists of several intervals, in each one the target has certain (positive or negative) acceleration and certain lateral acceleration ("turn magnitude"), both drawn randomly at the beginning of the interval.

%%%%%%%%%%%%%%%%%%%%%%%%%%%%%%%%%%%%%%%%%%%%%%%%%
%%%%%%%%%%%%%%%%%%%%%%%%%%%%%%%%%%%%%%%%%%%%%%%%%

\FloatBarrier
\section{Theoretical Analysis}
\label{app:theory}

%%%%%%%%%%%%%%%%%%%%%%%%%%%%%%%%%%%%%%%%%%%%%%%%%

\subsection{Non-linear Observation}
\label{app:linear_violation}

In this section, we discuss the relation between the theoretical analysis of \cref{prop:linear_violation} and the empirical results shown in \cref{fig:noise_params_toy}.
Then, we provide the proof of \cref{prop:linear_violation}.
% and discuss the corresponding empirical results shown in \cref{fig:noise_params_toy}.

\textbf{\cref{fig:noise_params_toy} vs.~\cref{prop:linear_violation}:}
\cref{fig:noise_params_toy} displays the noise parameters $\hat{R}$ learned by OKF in the toy problem.
In accordance with \cref{prop:linear_violation}, the noise $\sigma_D$ associated with Doppler is increased compared to the true measurement noise $R$. %in comparison to the positional noise.
In fact, not only $\sigma_D$ is increased, but also the positional variances are decreased, which is not explained by \cref{prop:linear_violation}.
This phenomenon origins in the absence of dynamics noise in this toy problem ($Q\equiv0$), which leads to scale-invariance w.r.t.~the absolute values of $\hat{R}$.
That is, if we multiply the whole matrix $\hat{R}$ by a constant factor, the filtering errors are unaffected.
Specifically, if we multiply $\hat{R}$ of \cref{fig:noise_R_OKF_toy} by a factor of $\approx 3$, the positional variances become aligned with those of \cref{fig:noise_R_KF_toy}, and $\sigma_D$ is increased by a factor of $\approx 13$ -- in accordance with \cref{prop:linear_violation}.
We repeated the tests with this modified $\hat{R}$, and indeed, the results were indistinguishable from the original OKF.

\begin{proof}[Proof of \cref{prop:linear_violation}]

Recall that in this problem, the KF applies the update step using an estimated observation model $\tilde{H} = H(\tilde{X})$:
$$ \tilde{H} = \begin{pmatrix} 1 \\ & 1 \\ && 1 \\ &&& \tilde{x}_x/\tilde{r} & \tilde{x}_y/\tilde{r} & \tilde{x}_z/\tilde{r} \end{pmatrix} . $$
Denoting the normalized estimation error $dx'=\frac{\tilde{x}}{\tilde{r}}-\frac{x}{r}$, we can rewrite $\tilde{H}$ as
\begin{equation*}
    \tilde{H} = H + 
    \left(\begin{smallmatrix}
    0 \\ & 0 \\ && 0 \\ &&& dx_x' & dx_y' & dx_z'
    \end{smallmatrix}\right) .
\end{equation*}
% where $dx'=\tilde{x}/\tilde{r}-x/r$ is the corresponding normalized estimation error. %Note that in the private case $\tilde{x}=z$, these errors coincide with the observation noise $\nu$ (up to normalization).
By shifting the observation model in \cref{eq:KF_model} from $H$ to $\tilde{H}$, and denoting the noise by $\nu=(\nu_x,\nu_y,\nu_z,\nu_D)^\top$, we receive
\begin{align*}
    Z=&H X +\nu = \tilde{H}X + \left(\begin{smallmatrix} \nu_x \\ \nu_y \\ \nu_z \\ \nu_D -dx_x'u_x-dx_y'u_y-dx_z'u_z \end{smallmatrix}\right) = \tilde{H}X + \left(\begin{smallmatrix} \nu_x \\ \nu_y \\ \nu_z \\ \nu_D - dx'\cdot u \end{smallmatrix}\right) ,
\end{align*}
where $u$ denotes the current target velocity.
We see that the effective observation noise is $\tilde{\nu}=Z-\tilde{H}X = (\nu_x, \nu_y, \nu_z, \nu_D - dx'\cdot u)^\top$.

To show that all the off-diagonal entries of $\tilde{R}=Cov(\tilde{\nu})$ vanish, recall that the estimation error $dx'$ is assumed to be independent of the velocity $u$.
According to \cref{eq:KF_model}, $\nu$ is also independent of $u$.
% $dx_x',\nu_x$ are both independent of the velocity $u_x$ (the former by the assumption, and the latter by the model of Eq.~\ref{eq:KF_model} in Definition~\ref{problem:toy_doppler}).
Hence, $Cov(dx_x'\cdot u_x,\ \nu_x)=E(dx_x'\cdot u_x \cdot \nu_x)=E(dx_x'\nu_x)E(u_x)$ which vanishes by symmetry ($E(u_x)=0$). The same result holds for coordinates $y,z$.
Thus, $\tilde{R}$ is diagonal.
Finally, by denoting $C=Var(dx'\cdot u)>0$ we have $Cov(\tilde{\nu}) = \tilde{R}$ as required.

Relying again on symmetry $E(u),E(dx')=0$, we can further calculate $C = Var(dx'\cdot u) = E(||dx'||^2)E(||u||^2) = \Omega(E(||u||^2))$, where $\Omega$ (``big-omega'') corresponds to an asymptotic lower bound.
\end{proof}

%%%%%%%%%%%%%%%%%%%%%%%%%%%%%%%%%%%%%%%%%%%%%%%%%

\subsection{Non-i.i.d Noise}
\label{app:iid_violation}

The assumption of i.i.d noise in \cref{assumption:KF} is violated in many practical scenarios.
Certain models with non-i.i.d noise can be solved analytically, if modeled correctly.
For example, if the noise is auto-regressive with a known order $p$, an adjusted KF model may consider the last $p$ values of the noise itself as part of the system state~\citep{KF_colored_noise}.
However, the actual noise model is often unknown or infeasible to solve analytically.

Furthermore, the violation of the i.i.d assumption may even go unnoticed.
We discuss a potential example in \cref{sec:experiments}, where the noise is i.i.d in \textit{spherical} coordinates -- but is not so after the transformation to \textit{Cartesian} coordinates.
To see that, consider a radar with noiseless angular estimation (i.e., only radial noise), and a low target ($x_z \approx 0$).
Clearly, most of the noise concentrates on the XY plane -- both in the current time-step and in the following ones (until the target moves away from the plane). Hence, the noise is statistically-dependent over time-steps.

We may formalize this intuition for the toy \cref{problem:toy_lidar}.
Denote the system state at time $t$ by $X_t=((X_t)_1,(X_t)_2)^\top$, and denote $\tan\theta_t=\frac{(X_t)_2}{(X_t)_1}$.
By transforming $R_{polar}$ of \cref{problem:toy_lidar} to Cartesian coordinates, the observation noise is drawn from the distribution $\nu_t\sim \mathcal{N}(0,R(\theta_t))$, where
\begin{equation}
\label{eq:R_theta}
    R(\theta)=\begin{pmatrix} r_0\cos^2(\theta) & r_0\cos(\theta)\sin(\theta) \\ r_0\cos(\theta)\sin(\theta) & r_0\sin^2(\theta) \end{pmatrix} .
\end{equation}
Since consecutive time steps are likely to have similar values of $\theta_t$, the noise $\nu_t$ is no longer independent across time steps.

The effect of this violation of the i.i.d assumption is analyzed in \cref{prop:iid_violation}, whose proof is provided below.
% explains the effect of the violation of the i.i.d assumption on the optimal parameters in the toy lidar \cref{problem:toy_lidar}.
% Below we provide its proof.
%
\begin{proof}[Proof of \cref{prop:iid_violation}]
\phantom{ }

\textbf{Noise estimation:}
First, notice that the whole setting of \cref{problem:toy_lidar} is invariant to the target direction $\theta$:
the initial state distribution is radial, and the motion noise $Q$ is isotropic.
Hence, for any target at any time-step, $\theta_t\sim\left[0,2\pi\right)$ is uniformly distributed.
By direct calculation,
\begin{align*}
    E_\theta\left[\hat{R}_N(KF)_{11}\right] &= E_\theta\left[r_0 \cos^2\theta\right] = \int_0^{2\pi} \frac{r_0}{2\pi} \cos^2\theta d\theta = \frac{r_0}{2} \\
    E_\theta\left[\hat{R}_N(KF)_{22}\right] &= E_\theta\left[r_0 \sin^2\theta\right] = \int_0^{2\pi} \frac{r_0}{2\pi} \sin^2\theta d\theta = \frac{r_0}{2} \\
    E_\theta\left[\hat{R}_N(KF)_{12}\right] &= E_\theta\left[\hat{R}_N(KF)_{21}\right] = E_\theta\left[r_0 \cos\theta\sin\theta\right] = 0 . % \\ &= \int_0^{2\pi} \frac{r_0}{2\pi} \cos\theta\sin\theta d\theta = 0 \\
\end{align*}
Since the targets in the data are i.i.d, the noise estimation of \cref{algo:KF} converges almost surely according to the law of large numbers, as required:
$$ \hat{R}_N(KF) \xrightarrow{\text{a.s.}} \hat{R}_{est} = \begin{pmatrix} r_0/2 & 0 \\ 0 & r_0/2 \end{pmatrix} . $$

\textbf{Optimization:}
We use again the radial symmetry and invariance to rotations in the problem:
w.l.o.g, we assume that the optimal noise covariance parameter is diagonal, i.e., $\hat{R}_{opt}(r) = \left(\begin{smallmatrix} r & 0 \\ 0 & r \end{smallmatrix}\right)$ for some $r>0$. Our goal is to find $r$, and in particular to compare it to $r_0/2$.

At a certain time $t$, where the system state is $X_t$, denote $E[X_t] = x_0 = (x_1,x_2)^\top$ and $Cov(X_t) = P_0 = \left(\begin{smallmatrix} p & 0 \\ 0 & p \end{smallmatrix}\right)$ (where $p>0$).
Denote the observation received at time $t$ by $z=(x_1+dx_1,x_2+dx_2)^\top$.
We are interested in the point-estimate $\hat{x}$ of the KF following the update step (\cref{fig:KF}).
By substituting $x_0$, $P_0$, the observation $z$ and the noise parameter $\hat{R}_{opt}(r)$ in the update step, we have
\begin{align*}
    \hat{x} &= x_0 + P_0H^\top(HP_0H^\top+\hat{R}_{opt}(r))^{-1}(z-Hx_0) %\\
    = x_0 + P_0(P_0+\hat{R}_{opt}(r))^{-1}(z-x_0) \\
    &= x_0 + \begin{pmatrix}
        \frac{p}{p+r} & 0 \\
        0 & \frac{p}{p+r}
    \end{pmatrix} \begin{pmatrix}
        dx_1 \\ dx_2
    \end{pmatrix} % \\
    = \begin{pmatrix}
        x_1+\frac{p}{p+r}dx_1 \\ x_2+\frac{p}{p+r}dx_2
    \end{pmatrix} .
\end{align*}
On the other hand, the \textit{true} observation noise covariance at time $t$ is $R(\theta_t)$ of \cref{eq:R_theta} (for the random variable $\theta_t$).
If we add the assumption that the state $X_t$ is normally distributed ($X_t \sim \mathcal{N}(x_0,P_0)$), and use the true noise covariance $R(\theta_t)$, then the update step of \cref{fig:KF} gives us the true posterior expected state:
\begin{align*}
    x_{true} &= x_0 + P_0(P_0+R(\theta_t))^{-1}(z-x_0) \\
    &= x_0 + \begin{pmatrix}
        \frac{r_0\sin^2\theta+p}{p+r_0} & -\frac{r_0\cos\theta\sin\theta}{p+r_0} \\
        -\frac{r_0\cos\theta\sin\theta}{p+r_0} & \frac{r_0\cos^2\theta+p}{p+r_0}
    \end{pmatrix} \begin{pmatrix}
        dx_1 \\ dx_2
    \end{pmatrix} \\ &= \begin{pmatrix}
        x_1 + \frac{(r_0\sin^2\theta+p)dx_1 - (r_0\cos\theta\sin\theta)dx_2}{p+r_0} \\
        x_2 + \frac{(r_0\cos^2\theta+p)dx_2 - (r_0\cos\theta\sin\theta)dx_1}{p+r_0}
    \end{pmatrix} .
\end{align*}

We can use the standard MSE decomposition for the point-estimate $x$, into the bias term of $x$ and the variance term of the state distribution: $MSE = MSE_{var}(P_{true}) + MSE_{bias}(x, x_{true})$.
Notice that $MSE_{var}(P_{true})$ is independent of our estimator, as it corresponds to the inherent uncertainty $P_{true}$ (defined by applying to $P_0$ the update step with the true covariance $R(\theta_t)$).
Thus, our objective is to minimize $MSE_{bias}(\hat{x}, x_{true}) = E[||\hat{x}-x_{true}||^2]$.
% The expected error of the point-estimate $x$ can be decomposed into two terms: the bias of $x$, and the variance of the system-state distribution. More formally, we can write $MSE = MSE_{var}(P_{true}) + MSE_{bias}(x, x_{true})$, where we can only control the latter.

For the calculation below, we denote $a(r) \coloneqq p/(p+r)$ and use the identity $\sin2\theta=2\cos\theta\sin\theta$.
In addition, from radial symmetry of $dx = z - x_0$ we have $E[dx_1^2]=E[dx_2^2]$ and $E[dx_i]=0$, thus we can denote $v \coloneqq Var(dx_i) = E[dx_i^2]$.
% according to $F$ in Definition~\ref{def:simplified_lidar}, $E\left[dx_1^2\right]=E\left[dx_2^2\right]=q$
%
\begin{align*}
    MSE_{bias}&(\hat{x}(a), x_{true}) = E||\hat{x}(a)-x_{true}||^2 \\
    =& E\left[
    \left((a-\frac{r_0\sin^2\theta+p}{p+r_0})dx_1 + \frac{r_0\sin(2\theta)/2}{p+r_0}dx_2\right)^2\right. \\
    &+ \left.\left((a-\frac{r_0\cos^2\theta+p}{p+r_0})dx_2 + \frac{r_0\sin(2\theta)/2}{p+r_0}dx_1\right)^2\right] \\
    %\right] \\
    =& E\left[
    dx_1^2\left( a^2 - 2a\frac{r_0\sin^2\theta+p}{p+r_0} + C_1 \right) \right. %\\
    + \frac{r_0^2\sin^2(2\theta)/4}{(p+r_0)^2}dx_2^2 + A_1dx_1dx_2 \\
    &+ dx_2^2\left( a^2 - 2a\frac{r_0\cos^2\theta+p}{p+r_0} + C_2 \right) %\\
    + \left. \frac{r_0^2\sin^2(2\theta)/4}{(p+r_0)^2}dx_1^2 + A_2dx_1dx_2
    \right] \\
    =& 2va^2 - 2va\frac{r_0+2p}{p+r_0} + v(C_1+C_2) + v\frac{r_0^2\sin^2(2\theta)/2}{(p+r_0)^2} ,
\end{align*}
where $C_{1,2}$ are independent of $a$, and $A_{1,2}$ are multiplied by $E[dx_1dx_2]=0$ and vanish.
To minimize we calculate
\begin{align*}
    0 &= \frac{\partial MSE_{bias}(\hat{x}(a), x_{true})}{\partial a} = 4v\cdot a - 2v\frac{2p+r_0}{p+r_0} ,
\end{align*}
which gives us
\begin{align*}
    a = \frac{p+r_0/2}{p+r_0} .
\end{align*}
Notice that $MSE_{bias}$ clearly diverges as $|a|\rightarrow \infty$, hence the only critical point necessarily corresponds to a minimum of the $MSE$.
Hence, the optimal $MSE$ is given when substituting the following $r$ in $\hat{R}_{opt}$:
\begin{align*}
    r &= p/a-p = \frac{p^2+pr_0 - (p^2+pr_0/2)}{p+r_0/2} = \frac{pr_0}{2p+r_0} .
\end{align*}
Finally, recall that $(\hat{R}_{est})_{ii}=r_0/2$ and compare to $r$ directly:
% \begin{align*}
%     (R_{est})_{ii} - (R_{opt})_{ii} &= r_0/2-r \\ &= \frac{(pr_0 + r_0^2/2) - pr_0}{2p+r_0} = \frac{r_0^2/2}{2p+r_0} > 0
% \end{align*}
\begin{align*}
    (\hat{R}_{est})_{ii} - (\hat{R}_{opt})_{ii} &= r_0/2-r = \frac{r_0^2/2}{2p+r_0} > 0 .
\end{align*}
\end{proof}

%%%%%%%%%%%%%%%%%%%%%%%%%%%%%%%%%%%%%%%%%%%%%%%%%
%%%%%%%%%%%%%%%%%%%%%%%%%%%%%%%%%%%%%%%%%%%%%%%%%

\FloatBarrier
\section{OKF: Extended Experiments}
\label{app:okf}

\subsection{Additional Scenarios and Baselines: A Case Study}
\label{app:okf_detailed}

In this section, we extend the experiments of \cref{sec:doppler} with a detailed case study.
The case study considers 5 types of tracking scenarios (\textit{benchmarks}) and 4 variants of the KF (\textit{baselines}) -- 20 experiments in total.
In each experiment, we compare the test MSE of OKF against the standard KF.
The experiments in \cref{sec:nkf} and \cref{sec:doppler} are 3 particular cases.
For each benchmark, we simulate 1500 targets for training and 1000 targets for testing.

\textbf{Benchmarks (scenarios):}
\cref{sec:experiments} discusses the sensitivity of \cref{algo:KF} to violations of \cref{assumption:KF}.
In this case study, we consider 5 benchmarks with different subsets of violations of \cref{assumption:KF}.
The \textit{Free Motion} benchmark is intended to represent a realistic Doppler radar problem, with targets and observations simulated as in \cref{sec:nkf}: each target trajectory consists of multiple segments of different turns and accelerations.
On the other extreme, the \textit{Toy} benchmark (\cref{problem:toy_doppler}) introduces multiple simplifications (as visualized in \cref{fig:radar_simplification}).
In the Toy benchmark, the only violation of \cref{assumption:KF} is the non-linear observation $H$, as discussed in \cref{sec:doppler}.
Note that \cref{sec:video} and \cref{sec:lidar} experiment with settings of a linear observation model.

We design 5 benchmarks within the spectrum of complexity between Toy and Free Motion.
Each benchmark is defined as a subset of the following properties, as specified in \cref{tab:benchmarks} and visualized in \cref{fig:trajectories}:
\begin{itemize}
    \item \textit{anisotropic}: horizontal motion is more likely than vertical (otherwise direction is distributed uniformly).
    % \vspace{-0.05cm}
    \item \textit{polar}: radar noise is generated i.i.d in spherical coordinates (otherwise noise is Cartesian i.i.d).
    \item \textit{uncentered}: targets are dispersed in different locations far from the radar (otherwise they are concentrated in the center).
    \item \textit{acceleration}: speed change is allowed (through intervals of constant acceleration).
    \item \textit{turns}: non-straight motion is allowed.
\end{itemize}
% An optimal analytic solution of each benchmark is hard to derive -- even if the benchmark assumptions are noticed and modeled correctly. However, in Section~\ref{sec:toy_analysis} we theoretically analyze the Toy benchmark, showing that its single assumption violation is sufficient to significantly modify the optimal parameters $R$.

\begin{table}%[b]
\centering
\caption{\small Benchmarks and the properties that define them. ``V'' means that the benchmark satisfies the property.}
\label{tab:benchmarks}
\setlength\tabcolsep{5pt}
\begin{tabular}{|c|ccccc|}
\hline
Benchmark & \rotatebox[origin=c]{60}{anisotropic} & \rotatebox[origin=c]{60}{polar} & \rotatebox[origin=c]{60}{uncentered} & \rotatebox[origin=c]{60}{acceleration} & \rotatebox[origin=c]{60}{turns} \\
\hline
Toy & O & O & O & O & O \\
Close & V & V & O & O & O \\
Const\_v & V & V & V & O & O \\
Const\_a & V & V & V & V & O \\
Free & V & V & V & V & V \\
\hline
\end{tabular}
\end{table}

\begin{figure}%[!h]
% \vspace{-10pt}
\centering
\begin{subfigure}{.19\textwidth}
  \centering
  \includegraphics[width=1.\linewidth]{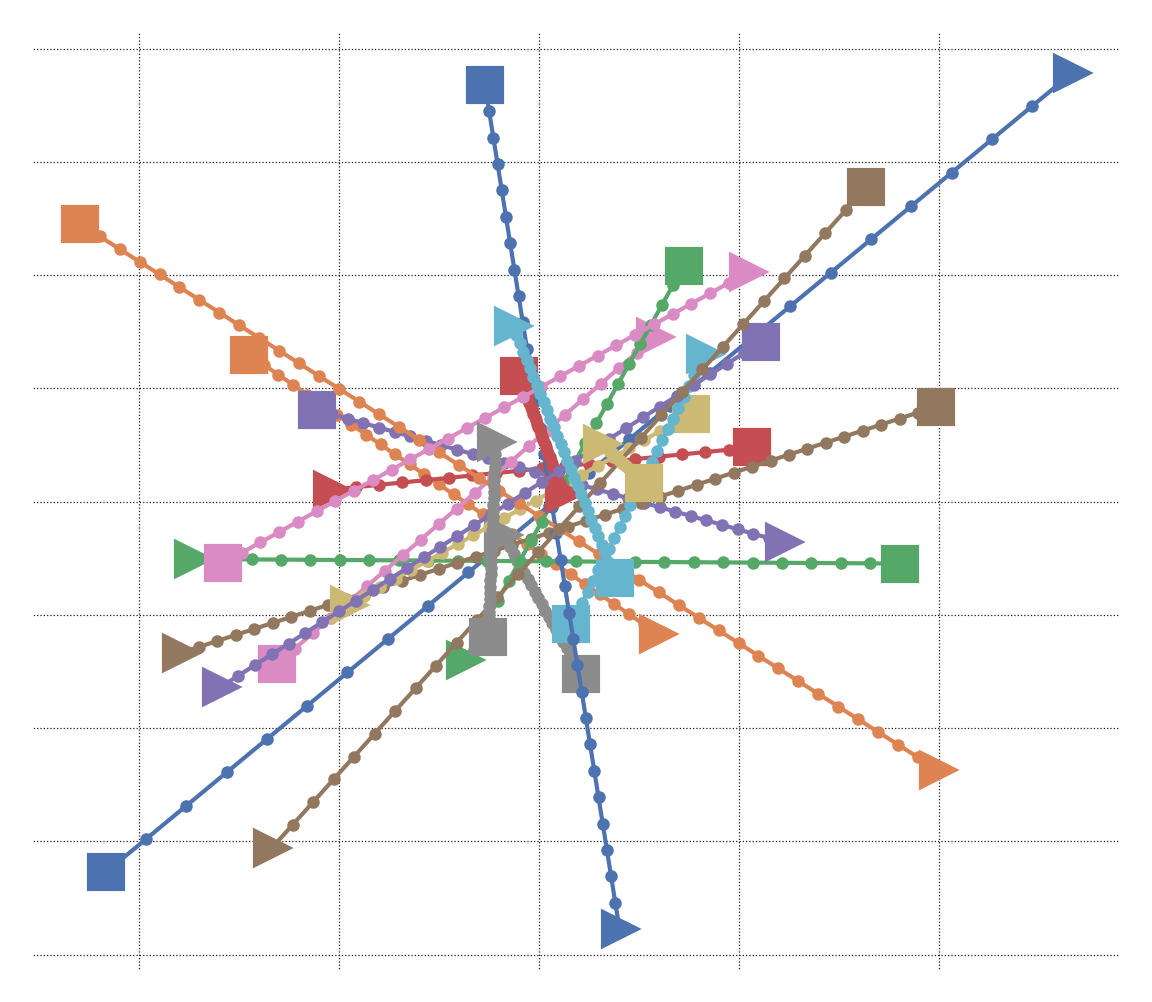}
  \caption{Toy}
\end{subfigure}
\begin{subfigure}{.19\textwidth}
  \centering
  \includegraphics[width=1.\linewidth]{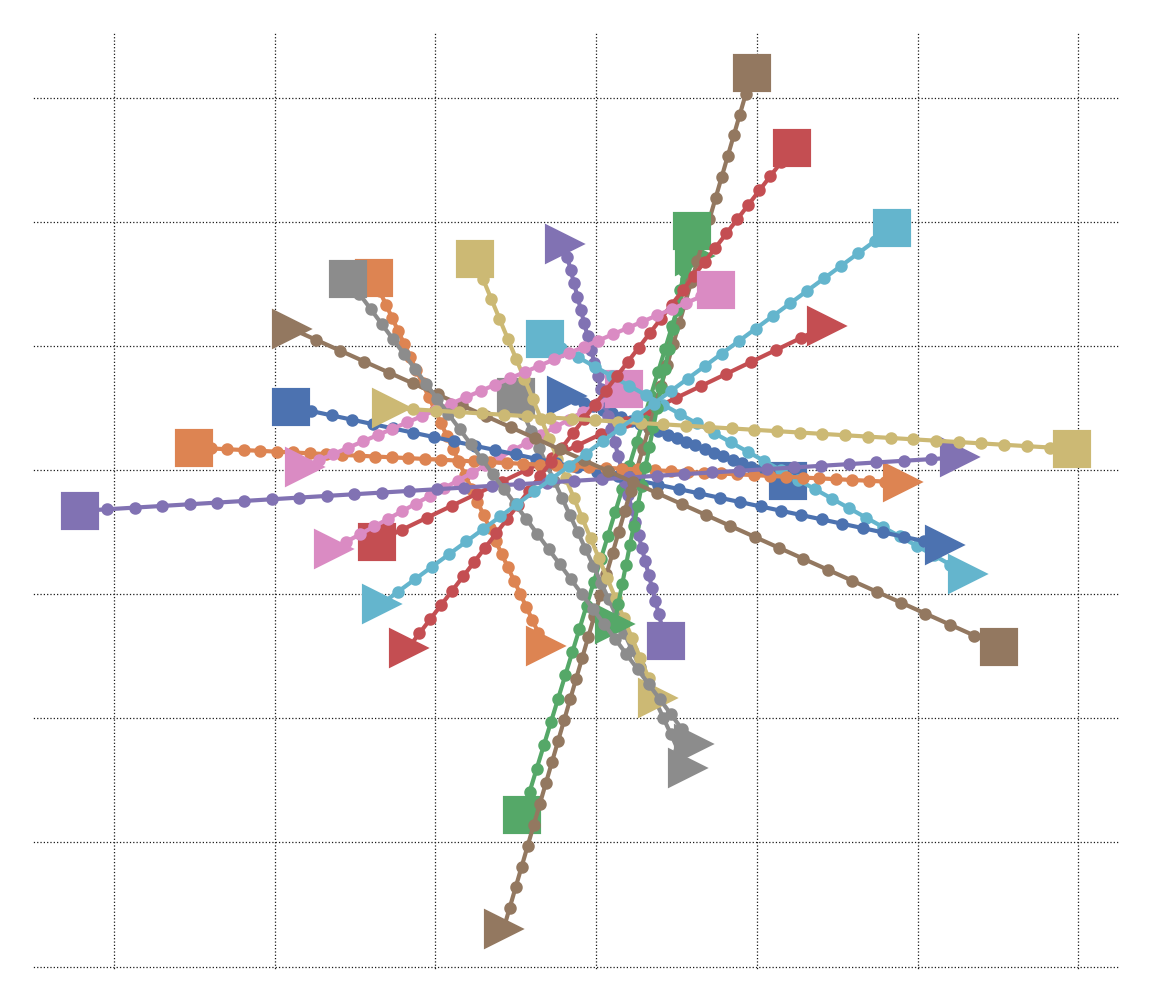}
  \caption{Close}
\end{subfigure}
\begin{subfigure}{.19\textwidth}
  \centering
  \includegraphics[width=1.\linewidth]{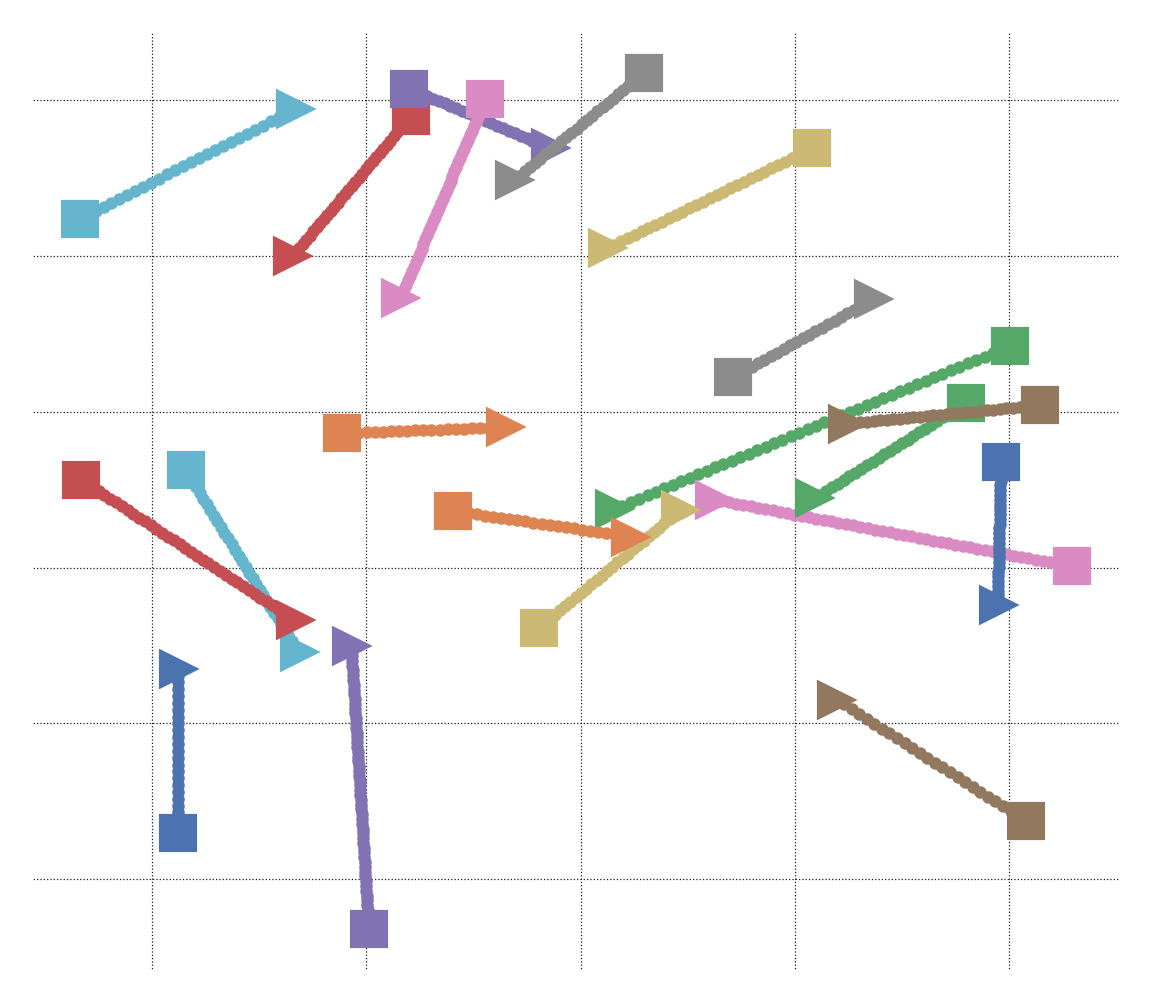}
  \caption{Const\_v}
\end{subfigure} %\\
\begin{subfigure}{.19\textwidth}
  \centering
  \includegraphics[width=1.\linewidth]{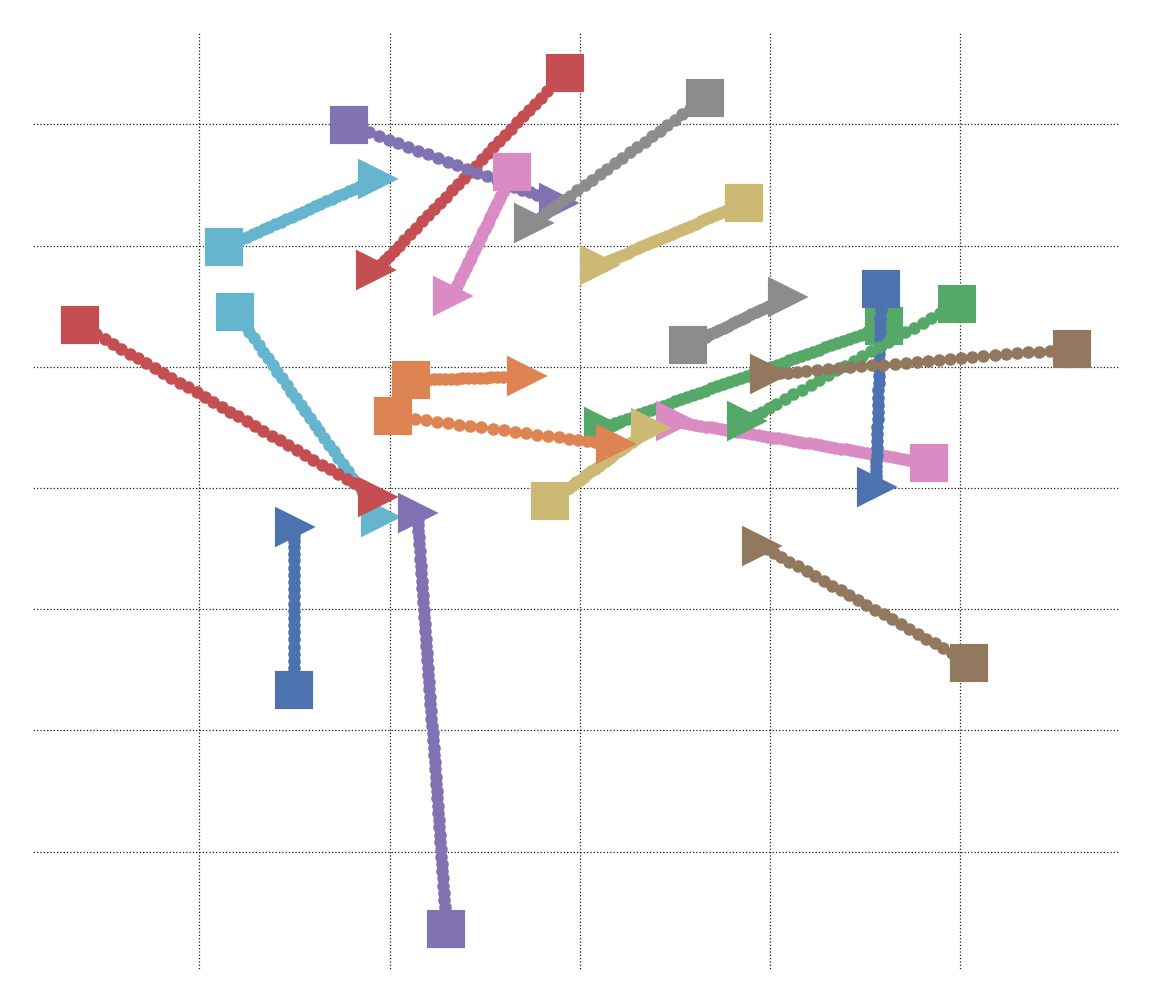}
  \caption{Const\_a}
  \label{fig:trajs_consta}
\end{subfigure}
\begin{subfigure}{.19\textwidth}
  \centering
  \includegraphics[width=1.\linewidth]{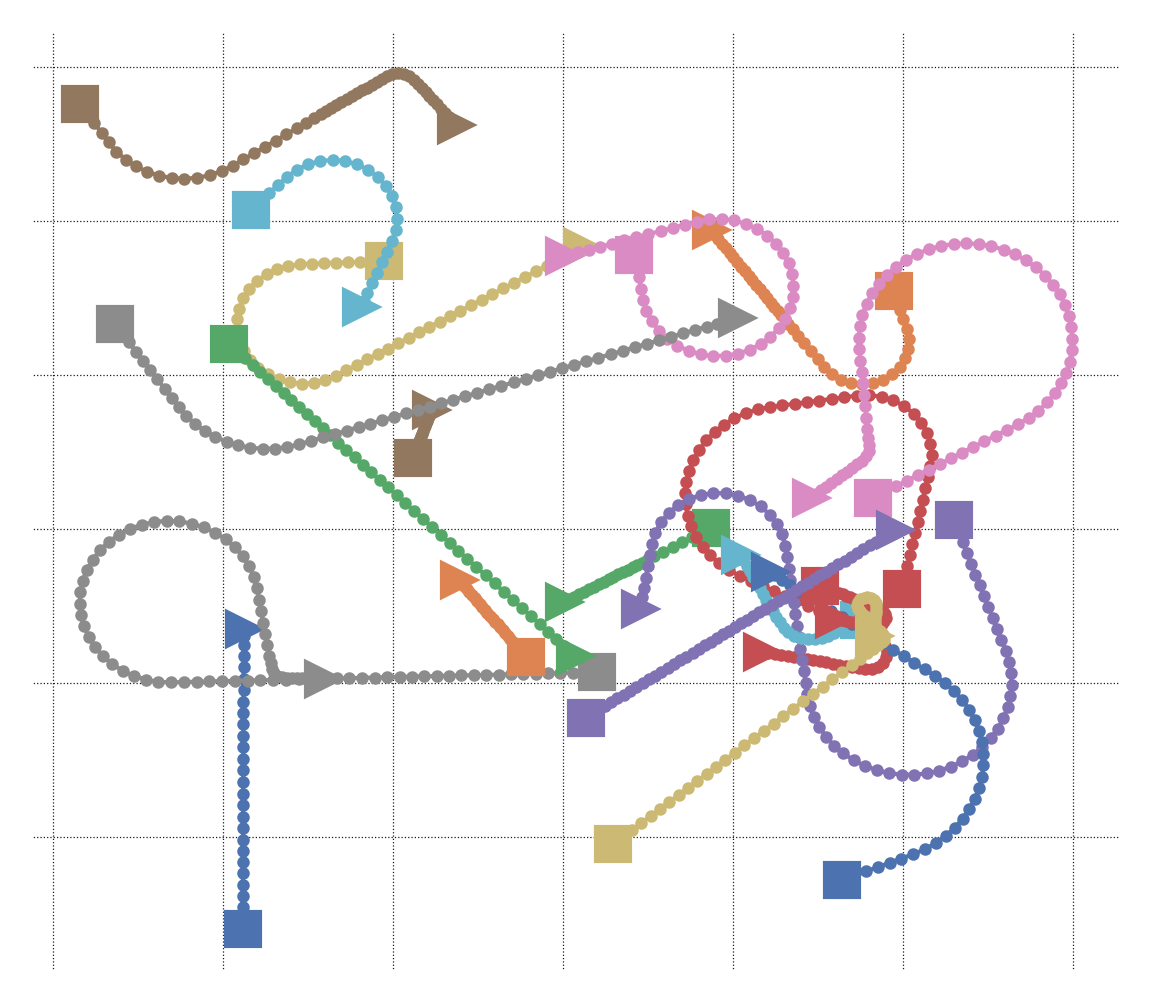}
  \caption{Free motion}
  \label{fig:trajs_free}
\end{subfigure}
\caption{\small Samples of targets trajectories in the various benchmarks, projected onto the XY plane.}
\label{fig:trajectories}
\end{figure}

% \begin{figure}[!h]
%   \centering
%   \includegraphics[width=1.\linewidth]{figures/benchmarks_EDA.png}
%     \caption{\small Descriptive statistics of the various benchmarks of Section~\ref{sec:OKF}: duration of targets trajectories; distance from the radar; targets speed; targets acceleration; motion direction (90 degrees correspond to horizontal motion); and observation errors vs. distance from radar.}
%     \label{fig:benchmarks_EDA}
% \end{figure}

% \begin{figure}%[!t]
% \centering
% \includegraphics[width=.5\linewidth]{figures/sim/episode_targets.jpg}
% \caption{\small A sample of targets trajectories in the Free-Motion benchmark in the radar problem.}
% \label{fig:sample_episode}
% \end{figure}

\textbf{Baselines (KF variants):}
All the experiments above compare OKF to the standard KF baseline.
In practice, other variants of the KF are often in use.
Here we define 4 such variants as different baselines to the experiments.
In each experiment, we compare the baseline tuned by \cref{algo:KF} to its Optimized version trained by \cref{algo:OKF} (denoted with the prefix ``O'' in its name).
For \cref{algo:OKF}, we use the Adam optimizer with a single training epoch over the 1500 training trajectories, 10 trajectories per training batch, and learning rate of 0.01.
The optimization was run for all baselines in parallel and required a few minutes per benchmark, on eight i9-10900X CPU cores in a single Ubuntu machine.

The different baselines are designed as follows.
\textit{EKF} baselines use the non-linear Extended KF model \citep{KF_theory}.
The EKF replaces the approximation $H \approx H(z)$ of \cref{sec:nkf} with $H \approx \nabla_x h(\hat{x})$, where $h(x)=H(x)\cdot x$ and $\tilde{x}$ is the current state estimate.
\textit{Polar} baselines (denoted with ``p'') represent the observation noise $R$ with spherical coordinates, in which the polar radar noise is i.i.d.

\begin{table*}%[b]
\centering
\caption{\small Test MSE results of \cref{algo:KF} and \cref{algo:OKF} over 5 benchmarks (scenarios) and 4 baselines (variants of KF).
For KFp we also consider an ``oracle'' baseline with perfect knowledge of the noise.}
\label{tab:res_OKF}
\setlength\tabcolsep{5pt}
\begin{tabular}{|c|cc|ccc|cc|cc|}
\hline
Benchmark & KF & OKF & KFp & KFp (oracle) & OKFp & EKF & OEKF & EKFp & OEKFp \\
\hline
Toy & 151.7 & 84.2 & 269.6 & -- & 116.4 & 92.8 & {\bf 79.4} & 123.0 & 109.1 \\
Close & 25.0 & 24.8 & 22.6 & 22.5 & {\bf 22.5} & 26.4 & 26.1 & 24.5 & 24.1 \\
Const\_v & 90.2 & 90.0 & 102.3 & 102.3 & {\bf 89.2} & 102.5 & 99.7 & 112.7 & 102.1 \\
Const\_a & 107.5 & 101.6 & 118.4 & 118.3 & {\bf 100.3} & 110.0 & 107.0 & 126.0 & 108.7 \\
Free & 125.9 & 118.8 & 145.6 & 139.3 & {\bf 117.9} & 135.8 & 121.9 & 149.3 & 120.0 \\
\hline
\end{tabular}
\end{table*}

\begin{figure}%[h]%[b]
% \vspace{-10pt}
\centering
\begin{subfigure}{.8\textwidth}
  \centering
  \includegraphics[width=1.\linewidth]{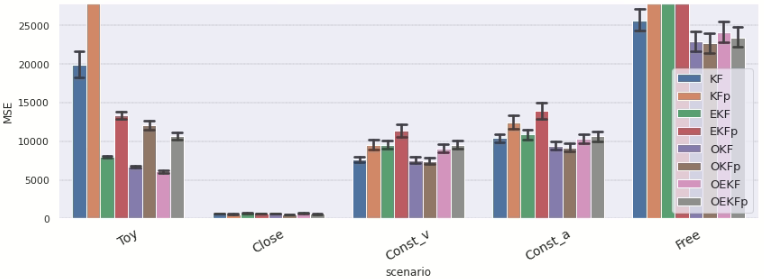}
  \caption{}
  \label{fig:res_all_KF}
\end{subfigure} \\
\begin{subfigure}{.45\textwidth}
  \centering
  \includegraphics[width=1.\linewidth]{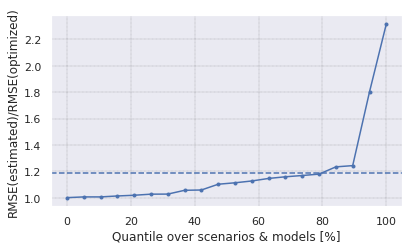}
  \caption{}
  \label{fig:res_relative}
\end{subfigure}
\caption{\small Summary of the test MSE of \cref{algo:KF} and \cref{algo:OKF} in different benchmarks (scenarios) and baselines. This is a different presentation of the results of \cref{tab:res_OKF}.
(a) also includes 95\% confidence intervals.
(b) shows, for each of the 20 experiments (5 benchmarks $\times$ 4 baselines), the MSE ratio between \cref{algo:KF} and \cref{algo:OKF}. We see that \cref{algo:OKF} wins in \textit{all} the experiments (ratio is always larger than 1) -- in some cases by large margins.
The dashed line represents the average MSE ratio over over all the experiments, showing an average advantage of $20\%$ to OKF.}
\label{fig:res_OKF_detailed}
\end{figure}

\textbf{Results:}
\cref{tab:res_OKF} summarizes the test errors (MSE) in all the experiments.
In each cell, the left column corresponds to the baseline \cref{algo:KF}, and the right to \cref{algo:OKF}.
In the model names, ``O'' stands for optimized, ``E'' for EKF and ``p'' for polar (or spherical).
The same results are also shown with confidence intervals in \cref{fig:res_OKF_detailed}.
Below we discuss the main findings.

\textbf{Choosing the KF configuration is not trivial:}
Consider the non-optimized KF baselines (left column in every cell in \cref{tab:res_OKF}).
In each benchmark, the results are sensitive to the baseline, i.e., to the choice of KF configuration -- $R$'s coordinates and whether to use EKF.
For example, in the Toy benchmark, EKF is the best design, since the observation model $H$ is non-linear.
In other benchmarks, however, the winning baselines may come as a surprise:
\begin{enumerate}
    \item Under non-isotropic motion direction (all benchmarks except Toy), EKF is worse than KF despite the non-linearity. It is possible that the horizontal prior reduces the stochasticity of $H$, making the derivative-based approximation unstable.
    % the advantage of EKF no longer justifies the instability of the derivative-based approximation.
    \item Even when the observation noise is spherical i.i.d, spherical representation of $R$ is not beneficial when targets are scattered far from the radar (last 3 benchmarks). It is possible that with distant targets, Cartesian coordinates have a more important role in expressing the horizontal prior of the motion.
\end{enumerate}
Since the best KF variant per benchmark seems hard to predict in advance, a practical system cannot rely on choosing the KF variant optimally -- and should rather be robust to this choice.
% While our interpretation of these results is open to debate, it would definitely be nontrivial to predict them in advance. In particular, in the more complex benchmarks, the default KF is the preferred model among the ones tuned by noise estimation.

\textbf{OKF is more accurate \textit{and} more baseline-robust:}
For \textit{every} benchmark and \textit{every} baseline (20 experiments in total), OKF (right column) outperformed noise estimation (left column).
% Note that OKF wins even in the Toy scenario, under the slightest violation of KF assumptions. %, and a fortiori wins in the more complex scenarios.
In addition, the variance between the baselines reduces under optimization, i.e., OKF makes the KF more robust to the selected configuration. % (which is also evident in Figure~\ref{fig:res_all_KF} in the supplementary material).

\textbf{OKF outperforms an \textit{oracle} baseline:}
We designed an ``oracle'' KF baseline -- with perfect knowledge of the observation noise covariance $R$ in spherical coordinates. We used it for all benchmarks except for Toy (in which the radar noise is not generated in spherical coordinates).
Note that in the constant-speed benchmarks (Close and Const\_v), $Q=0$ and is estimated quite accurately; hence, in these benchmarks the oracle has a practically perfect knowledge of both noise covariances.
Nevertheless, the oracle yields very similar results to \cref{algo:KF}.
This indicates that \textbf{the benefit of OKF is not in a better estimation accuracy of $Q$ and $R$, but rather in optimizing the desired objective}.

%%%%%%%%%%%%%%%%%%%%%%%%%%%%%%%%%%%%%%%%%%%%%%%%%

\FloatBarrier
\subsection{Sensitivity to Train Dataset Size}
\label{app:okf_train_size}

Each benchmark in the case-study of \cref{app:okf_detailed} has 1500 targets in its train data.
One may argue that numeric optimization may be more sensitive to smaller datasets than noise estimation; and even more so, when taking into account that the optimization procedure "wastes" a portion of the train data as a validation set.

In this section we test this concern empirically, by repeating some of the experiments of \cref{app:okf_detailed} with smaller subsets of the train datasets -- beginning from as few as 20 training trajectories.
\cref{fig:train_size} shows that the advantage of OKF over KF holds consistently for all sizes of train datasets, although it is indeed increases with the size.
Interestingly, in the Free Motion benchmark, \textbf{the test errors of KF and KFp \textit{increase} with the amount of train data!}

\begin{figure}[!h]
\centering
    \includegraphics[width=0.9\linewidth]{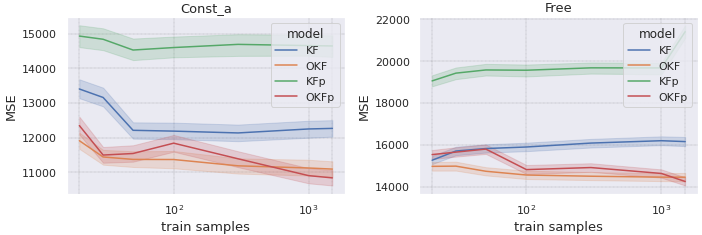}
\caption{\small The advantage of OKF over KF holds consistently for all sizes of train datasets -- including as small datasets as 20 trajectories. The shadowed areas correspond to 95\% confidence intervals.}
\label{fig:train_size}
\end{figure}

%%%%%%%%%%%%%%%%%%%%%%%%%%%%%%%%%%%%%%%%%%%%%%%%%

\FloatBarrier
\subsection{Generalization: Sensitivity to Distributional Shifts}
\label{app:okf_generalization}

In \cref{app:okf_detailed}, we demonstrate the robustness of OKF in different tracking scenarios: in every benchmark, OKF outperformed the standard KF over {\bf out-of-sample test data}.
This means that OKF did not overfit the noise in the training data.
What about \textbf{out-of-distribution test data}?
OKF learns patterns from the specific distribution of the train data -- how well will it generalize to different distributions?

\cref{sec:nkf} already addresses this question to some extent, as OKF outperformes both KF and NKF over out-of-distribution target accelerations (affecting both speed changes and turns radius).
In terms of \cref{eq:KF_model}, the modified acceleration corresponds to different magnitudes of motion noise $Q$; that is, we change the noise \textit{after} OKF optimized the noise parameters. Yet, OKF adapted to the change without further optimization, with better accuracy than the standard KF.
Thus, the results of \cref{sec:nkf} already provide a significant evidence for the robustness of OKF to certain distributional shifts.

\begin{figure}%[!h]
\centering
\begin{subfigure}{\textwidth}
    \includegraphics[width=\linewidth]{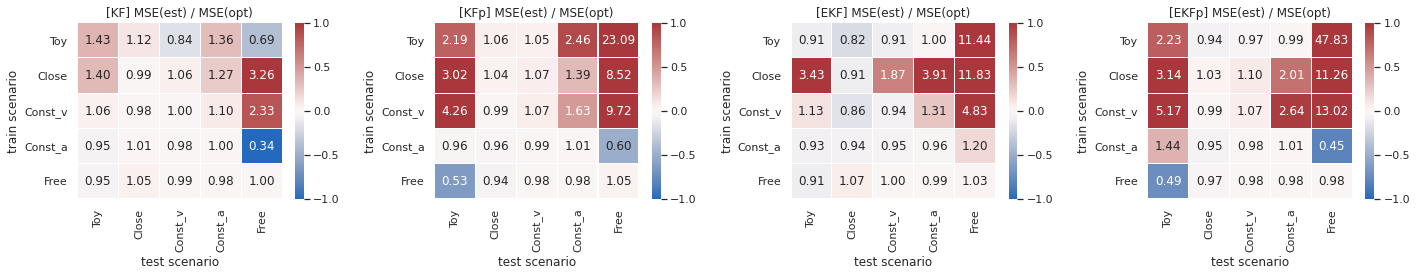}
    \caption{\small $MSE\_ratio = MSE(KF) / MSE(OKF)$ for every KF-baseline (KF,KFp,EKF,EKFp defined in \cref{app:okf_detailed}), and for every pair of train-scenario and test-scenario. The colormap scale is logarithmic ($\propto log(MSE\_ratio)$), where red values represent advantage to OKF ($MSE\_ratio>1$).}
    \label{fig:cross_scenarios_detailed}
\end{subfigure}
\begin{subfigure}{.5\textwidth}
    \includegraphics[width=\linewidth]{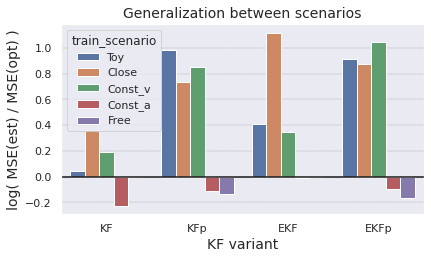}
    \caption{\small For every train-scenario, $MSE\_ratio$ is averaged over all the test-scenarios and is shown in a logarithmic scale. Positive values indicate advantage to OKF.}
    \label{fig:cross_scenarios_summary}
\end{subfigure}
\caption{\small Generalization tests: OKF vs. KF under distributional shifts between scenarios.}
\label{fig:cross_scenarios}
\end{figure}

In this section, we present a yet stronger evidence for the robustness of OKF -- not over a parametric distributional shift, but over {\bf entirely different benchmarks}. Specifically, we consider the 5 benchmarks (or scenarios) of \cref{app:okf_detailed}. For every pair (train-scenario, test-scenario), we train both KF and OKF on data of the train-scenario, then test them on data of the test-scenario.
For every such pair of scenarios, we measure the generalization advantage of OKF over KF through $MSE\_ratio = MSE(KF) / MSE(OKF)$ (where $MSE\_ratio>1$ indicates advantage to OKF). To measure the total generalization advantage of a model trained on a certain scenario, we calculate the geometric mean of $MSE\_ratio$ over all the test-scenarios (or equivalently, the standard mean over the logs of the ratios). The logarithmic scale guarantees a symmetric view of this metric of ratio between two scores.

This test is quite noisy, since a model optimized for a certain scenario may legitimately be inferior in other scenarios. Yet, considering all the results together in \cref{fig:cross_scenarios}, it is evident that OKF provides more robust models: it generalizes better in most cases, sometimes by a large margin; and loses only in a few cases, always by a small margin.

%%%%%%%%%%%%%%%%%%%%%%%%%%%%%%%%%%%%%%%%%%%%%%%%%

\FloatBarrier
\subsection{Ablation Test: Diagonal Optimization}
\label{sec:diagonal}

The main challenge in \cref{sec:okf} and \cref{algo:OKF} is to apply standard numeric optimization while keeping the symmetric and positive-definite constraints (SPD) of the parameters $Q,R$. To that end, we apply the Cholesky parameterization to $Q$ and $R$. In this section, we study the importance of this parameterization via an ablation test.

For the ablation, we define a naive version of OKF with a diagonal parameterization of $Q$ and $R$.
Such parameterization is common in the literature: \say{since both the covariance matrices must be constrained to be positive semi-definite, $Q$ and $R$ are often parameterized as diagonal matrices} \citep{noise_cov_estimation}.
We denote the diagonal variant by \textit{DKF}, and test it on all 5 Doppler benchmarks of \cref{app:okf_detailed}.

\cref{tab:ablation} displays the results.
In the first 3 benchmarks, where the target motion is linear, DKF is indistinguishable from OKF.
However, in the 2 non-linear benchmarks (which are also the same benchmarks used in \cref{sec:nkf}), DKF is inferior to OKF, though still outperforms the standard KF.

\begin{table*}%[b]
\centering
\caption{\small Ablation test: Diagonal optimized KF, tested on the 5 benchmarks of \cref{app:okf_detailed}.}
\label{tab:ablation}
\setlength\tabcolsep{5pt}
\begin{tabular}{|c|ccc|}
\hline
Benchmark & KF & DKF & OKF \\
\hline
Toy & 151.7 & 84.2 & 84.2 \\
Close & 25.0 & 24.8 & 24.8 \\
Const\_v & 90.2 & 90.1 & 90.0 \\
Const\_a & 107.5 & 106.1 & 101.6 \\
Free & 125.9 & 121.1 & 118.8 \\
\hline
\end{tabular}
\end{table*}

%%%%%%%%%%%%%%%%%%%%%%%%%%%%%%%%%%%%%%%%%%%%%%%%%

\FloatBarrier
\subsection{Video Tracking: Dataset License}
\label{sec:mot20_license}

The \href{https://motchallenge.net/data/MOT20/}{MOT20 video dataset} \citep{MOT20} is available under \textit{Creative Commons Attribution-NonCommercial-ShareAlike 3.0} License.
In \cref{sec:video}, we used the videos MOT20-01, MOT20-02 and MOT20-03 for training, and MOT20-05 for testing.

%%%%%%%%%%%%%%%%%%%%%%%%%%%%%%%%%%%%%%%%%%%%%%%%%

\FloatBarrier
\subsection{Lidar-based State Estimation: Visualization}
\label{sec:lidar_detailed}

\cref{fig:lidar_vis} visualizes a sample of the simulated trajectories and the model predictions in the lidar-based state estimation of \cref{sec:lidar}.

\begin{figure}[h]
\centering
\begin{subfigure}{.28\linewidth}
    \includegraphics[width=\linewidth]{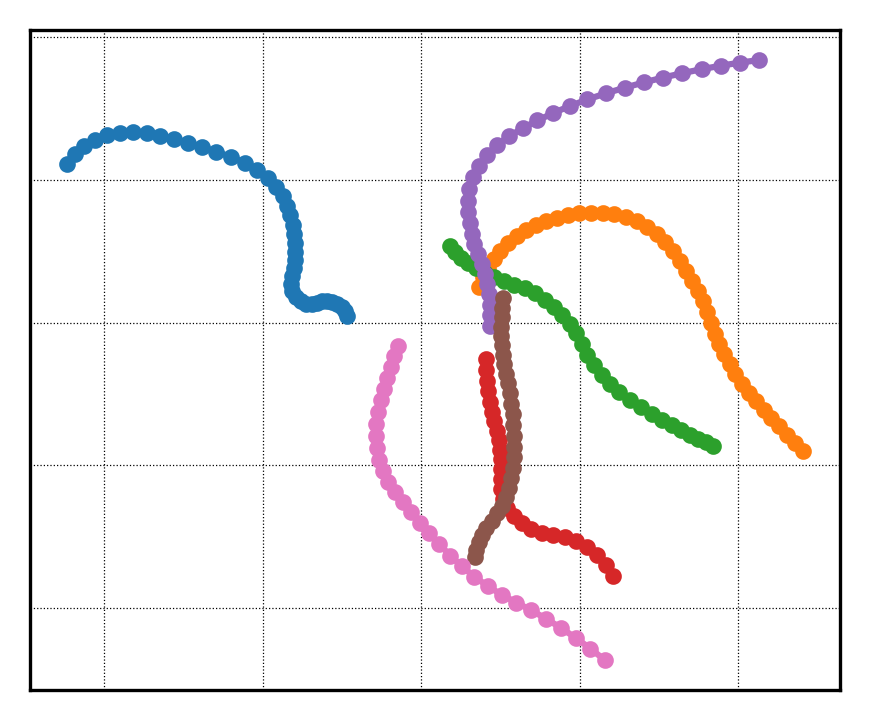}
    \caption{} %Sample trajectories\\ \phantom{X}}
    \label{fig:lidar_trajectories}
\end{subfigure}
\begin{subfigure}{.30\linewidth}
    \includegraphics[width=\linewidth]{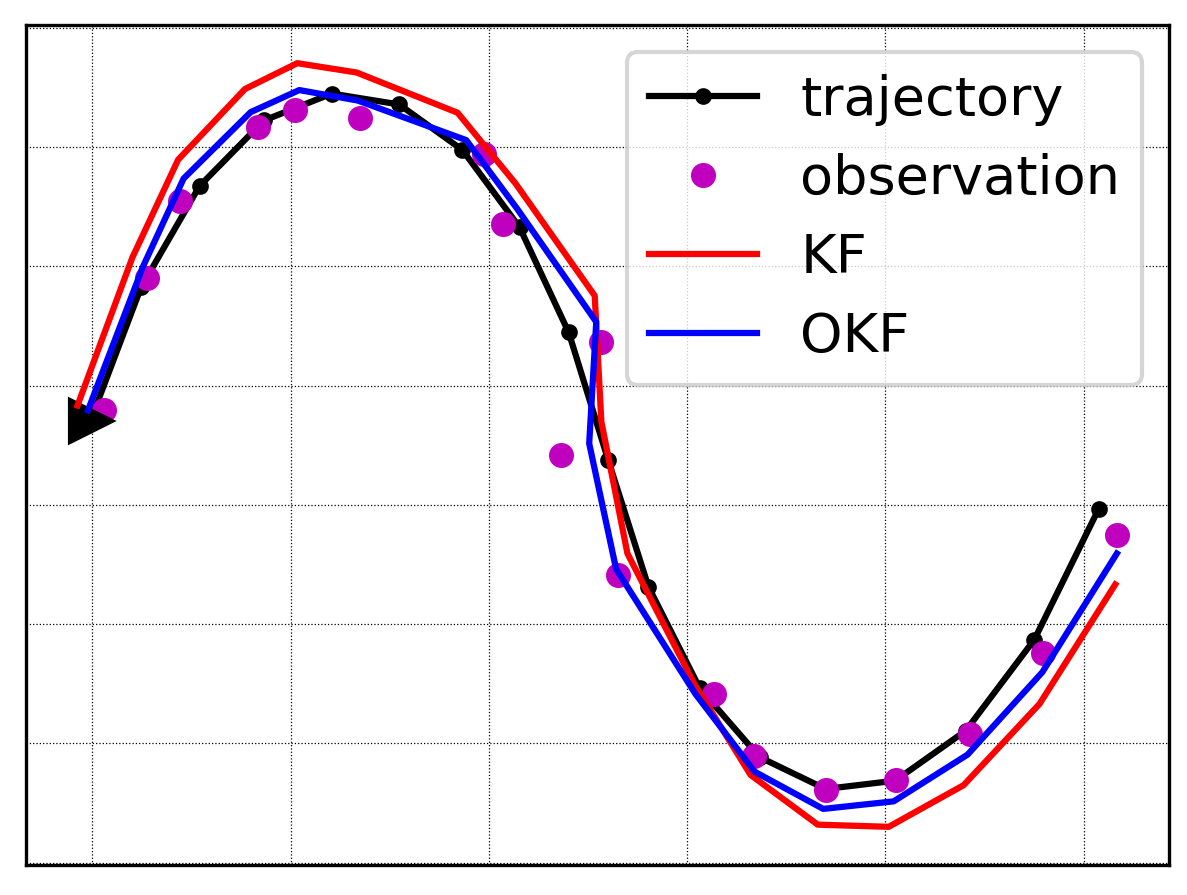}
    \caption{} %Turn segments in a sample trajectory + predictions}
    \label{fig:lidar_sample}
\end{subfigure} \\
% \begin{subfigure}{.66\linewidth}
%     \centering
%     \includegraphics[width=0.8\linewidth]{figures/lidar/lidar_SE.png}
%     \caption{Estimation errors (dashed=mean)}
%     \label{fig:lidar_MSE}
% \end{subfigure}
\caption{\small (a) A sample of simulated self-driving trajectories. (b) Segments of turns  within a sample trajectory, and the corresponding lidar-based estimations.}
% The lidar benchmark: (a) A sample of simulated self-driving trajectories. (b) Turn segments within a sample trajectory, and the corresponding predictions. (c) Estimation errors of KF and OKF.}
\label{fig:lidar_vis}
\end{figure}

%%%%%%%%%%%%%%%%%%%%%%%%%%%%%%%%%%%%%%%%%%%%%%%%%
%%%%%%%%%%%%%%%%%%%%%%%%%%%%%%%%%%%%%%%%%%%%%%%%%

\FloatBarrier
\section{Neural KF: Extended Discussion and Experiments}
\label{app:nkf}

%%%%%%%%%%%%%%%%%%%%%%%%%%%%%%%%%%%%%%%%%%%%%%%%%

\textbf{Preliminaries -- RNN and LSTM:}
% \textit{Neural networks} (NN) are parametric functions, usually constructed as a sequence of matrix-multiplications with some non-linear differentiable transformation between them. NNs are known to be able to approximate complicated functions, given that the right parameters are chosen. Optimization of the parameters of NNs is a field of vast research for decades, and usually relies on gradient-based methods, that is, calculating the errors of the NN with relation to some training-data of inputs and their desired outputs, deriving the errors with respect to the network's parameters, and moving the parameters against the direction of the gradient.
\textit{Recurrent neural networks} (RNN)~\citep{RNN} are neural networks that are intended to be iteratively fed with sequential data samples, and that pass information (the \textit{hidden state}) over iterations. Every iteration, the hidden state is fed to the next copy of the network as part of its input, along with the new data sample.
\textit{Long Short Term Memory} (LSTM)~\citep{LSTM} is an architecture of RNN that is particularly popular due to the linear flow of the hidden state over iterations, which allows to capture memory for relatively long term.
The parameters of a RNN are usually optimized in a supervised manner with respect to a training dataset of input-output pairs.

%%%%%%%%%%%%%%%%%%%%%%%%%%%%%%%%%%%%%%%%%%%%%%%%%

\textbf{Neural Kalman Filter:}
We introduce the Neural Kalman Filter (NKF), which incorporates an LSTM model into the KF framework.
The framework provides a probabilistic representation (rather than point estimate) and a separation between the prediction and update steps.
The LSTM is an architecture of recurrent neural networks, and is a key component in many SOTA algorithms for non-linear sequential prediction~\citep{process_prediction_review}. We use it for the non-linear motion prediction.

As shown in \cref{fig:NKF_diagram}, NKF uses separate LSTM networks for prediction and update steps.
In the prediction step, the target \textit{acceleration} is predicted on top of the linear motion model, instead of predicting the state directly. This regularized formulation is intended to express our domain knowledge about the kinematic motion of physical targets.
% Other variants of NKF were also tested, as discussed in \cref{app:nkf}.
% We made honest efforts to engineer the best NKF architecture for the problem; nevertheless, our methodological insight holds regardless of the technical quality of NKF.

\begin{figure}[!h]
\centering
\includegraphics[width=0.55\linewidth]{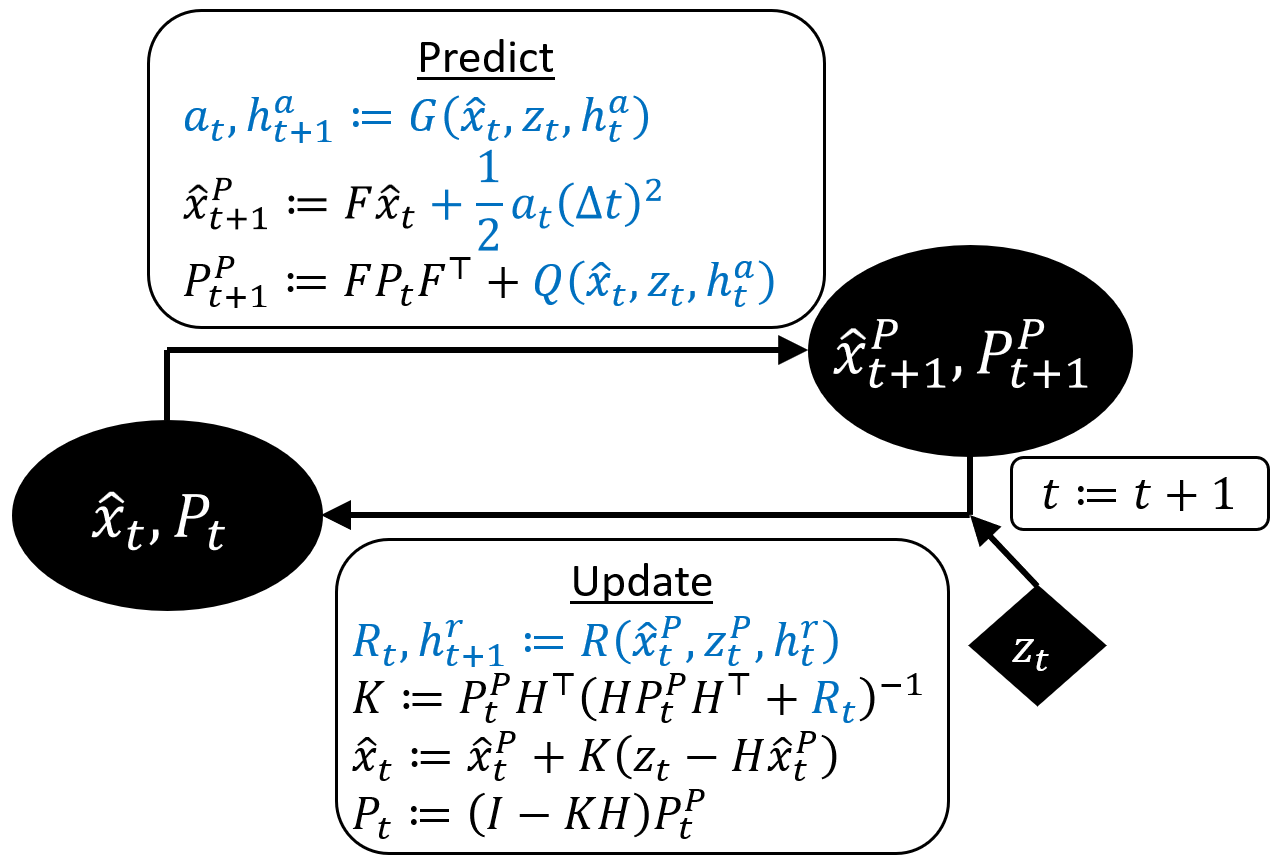}
\caption{\small The Neural Kalman Filter (NKF). Differences from \cref{fig:KF} are highlighted.
$\Delta t$ is constant; $G,Q$ are the outputs of an LSTM network with hidden state $h_a$; and $R$ is the output of an LSTM with hidden state $h_r$.
% In addition to the raw $x_t,z_t$, the networks are also fed with certain manually-crafted features (e.g., the projection of $z_t-x_t$ on the perpendicular direction $x_t^\perp$).
}
\label{fig:NKF_diagram}
\end{figure}

%%%%%%%%%%%%%%%%%%%%%%%%%%%%%%%%%%%%%%%%%%%%%%%%%

\textbf{Extended experiments:}
We extend the experiments of \cref{sec:nkf} with additional versions of NKF:
\begin{itemize}
    \item Predicted-acceleration KF (\textbf{aKF}): a variant of NKF that predicts the acceleration but not the covariances $Q$ and $R$.
    \item Neural KF (\textbf{NKF}): the model used in \cref{sec:nkf} and illustrated in \cref{fig:NKF_diagram}.
    \item Neural KF with H-prediction (\textbf{NKFH}): a variant of NKF that also predicts the observation model $H$ in every step.
\end{itemize}
In addition, while we still train with MSE loss, we add the test metric of Negative-Log-Likelihood (NLL) -- of the true state w.r.t the estimated distribution. Note that the NLL has an important role in the multi-target matching problem (which is out of the scope of this work).

% We extend the results of \cref{sec:nkf} with further experiments.
% First, we add a second variant of the Doppler problem, where the targets do not turn, but may still have acceleration (benchmark Const\_a of \cref{app:okf_detailed}).
% Second, while we still train with MSE loss, we add the test metric of Negative-Log-Likelihood (NLL) -- of the true state w.r.t the estimated distribution. Note that the NLL has an important role in the multi-target matching problem (which is out of the scope of this work).
% Third, we test 3 different versions of NKF:
% % To show that these results are not exclusive to the experiments of Section~\ref{sec:NKF}, we present extended experiments in this section.
% % Specifically, we consider 2 different benchmarks -- the free-motion benchmark of Section~\ref{sec:NKF} and Const\_a benchmark of Section~\ref{sec:OKF} (in which the targets may have acceleration but not turns). We also consider 3 different neural models:
% \begin{itemize}
%     \item Predicted-acceleration KF (\textbf{aKF}): a variant of NKF that predicts the acceleration but not the covariances $Q$ and $R$.
%     \item Neural KF (\textbf{NKF}): the model used in \cref{sec:nkf} and illustrated in \cref{fig:NKF_diagram}.
%     \item Neural KF with H-prediction (\textbf{NKFH}): a variant of NKF that also predicts the observation model $H$ in every step.
% \end{itemize}

For each benchmark and each model, we train the model on train data with a certain range of targets acceleration (note that acceleration affects both speed changes and turns sharpness), and tested it on targets with different acceleration ranges, some of them account for distributional shifts.
For each model we train two variants -- one with Cartesian representation of the observation noise $R$, and one with spherical representation (as in the baselines of \cref{app:okf_detailed}) -- and we select the one with the higher validation MSE (where the validation data is a portion of the data assigned for training).

\begin{figure}%[!h]
\centering
\begin{subfigure}{0.9\textwidth}
    \includegraphics[width=\linewidth]{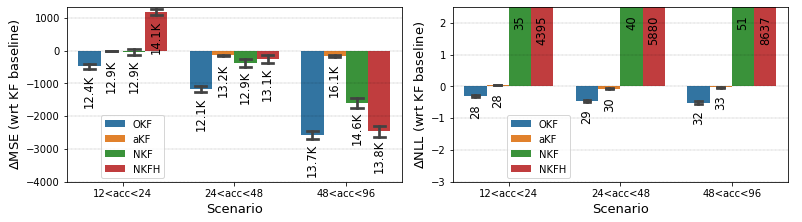}
    \caption{\small Free-motion benchmark}
    \label{fig:NKF_extended_free}
\end{subfigure}
\begin{subfigure}{0.9\textwidth}
    \includegraphics[width=\linewidth]{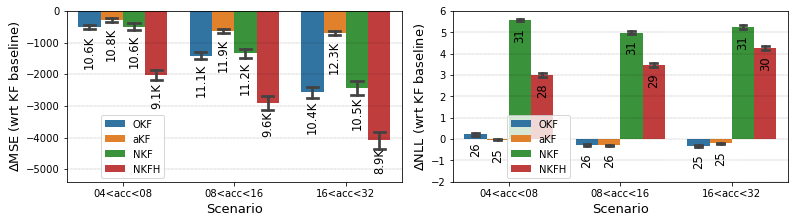}
    \caption{\small Const\_a benchmark (no turns)}
    \label{fig:NKF_extended_consta}
\end{subfigure}
\caption{\small The \textit{relative} MSE and NLL results of various models in comparison to the standard KF model. The textual labels specify the \textit{absolute} MSE and NLL. Note that certain bars of NLL are of entirely different scale and thus are cropped in the figure (their values can be seen in the labels).
In each benchmark, the models were trained with relation to MSE loss, on train data of the middle acceleration-range: the two other acceleration ranges in each benchmark correspond to generalization over distributional shifts.}
\label{fig:NKF_extended}
\end{figure}

\cref{fig:NKF_extended_free} shows that in the free-motion benchmark, all the 3 neural models improve the MSE in comparison to the standard KF, yet are outperformed by OKF.
Furthermore, while OKF has the best NLL, the more complicated models NKF and NKFH increase the NLL in orders of magnitude. %This issue may be handled by explicitly optimizing the NLL in addition to the MSE, but such multi-loss optimization is out of the scope of this work.
Note that the instability of NKFH is expressed in poor generalization to lower accelerations in addition to the extremely high NLL score.

\cref{fig:NKF_extended_consta} shows that in Const\_a benchmark, all the 3 neural models improve the MSE in comparison to the standard KF, but only NKFH improves in comparison to OKF as well.
On the other hand, NKFH still suffers from very high NLL.
% Note that while NKFH does better in this case than in the free-motion benchmark, it still suffers from very high NLL.

In summary, all 3 variants of NKF outperform the standard KF in both benchmarks in terms of MSE.
However, when comparing to OKF instead, aKF and NKF become inferior, and the comparison between NKFH and OKF depends on the selected benchmark and metric.

%%%%%%%%%%%%%%%%%%%%%%%%%%%%%%%%%%%%%%%%%%%%%%%%%
%%%%%%%%%%%%%%%%%%%%%%%%%%%%%%%%%%%%%%%%%%%%%%%%%

\end{document}